\documentclass{article}

\usepackage{amsfonts}
\usepackage{amssymb}
\usepackage{amsmath}

\usepackage{ver
batim}
\usepackage{microtype}
\usepackage{graphicx}
\usepackage{subfigure}
\usepackage{booktabs} 
\usepackage{comment}
\usepackage{xcolor}
\usepackage{bbm}
\usepackage{stmaryrd}
\usepackage{mathtools}
\usepackage{natbib}
\usepackage{verbatim}
\usepackage{comment}
\usepackage[linesnumbered, ruled,vlined]{algorithm2e}
\usepackage{algorithm,algorithmic}
\usepackage{enumitem}
\setlist[itemize]{topsep=0pt, leftmargin=3mm}

\usepackage{microtype}
\usepackage{graphicx}
\usepackage{subfigure}
\usepackage{booktabs} 
\usepackage{hyperref}
\usepackage{wrapfig}
\usepackage{graphicx}
\usepackage{enumitem}

\usepackage{tikz}
\usetikzlibrary{intersections,shapes.arrows}

\definecolor{richardcolor}{rgb}{0.9, 0.3, 0.1}

\usepackage{amsthm,amsmath,amssymb, amsfonts}
\newtheorem{theorem}{Theorem}[section]
\newtheorem{lemma}[theorem]{Lemma}

\newtheorem{definition}[theorem]{Definition}

\allowdisplaybreaks

\newcommand{\parens}[1]{\left( #1 \right)}
\newcommand{\ind}[1]{\mathbbm{1}\parens{#1}}

\newcommand*{\pr}[2][]{\mathbb{P}\ifx\\\left[#1\right]\\\else_{#1}\fi \left[#2\right]}

\newcommand*{\var}[2][]{\text{Var}\ifx\\\left[#1\right]\\\else_{#1}\fi \left[#2\right]}
\usepackage{hyperref}

\usepackage[accepted]{icml2020}

\icmltitlerunning{Stochastic Flows and Geometric Optimization on the Orthogonal Group}

\begin{document}

\twocolumn[
\icmltitle{Stochastic Flows and Geometric Optimization on the Orthogonal Group}



\icmlsetsymbol{equal}{*}

\begin{icmlauthorlist}
\icmlauthor{Krzysztof Choromanski}{go,co}
\icmlauthor{David Cheikhi}{equal,coo}
\icmlauthor{Jared Davis}{equal,go}
\icmlauthor{Valerii Likhosherstov}{equal,ca}
\icmlauthor{Achille Nazaret}{equal,coo}
\icmlauthor{Achraf Bahamou}{equal,co}
\icmlauthor{Xingyou Song}{equal,go}
\icmlauthor{Mrugank Akarte}{co}
\icmlauthor{Jack Parker-Holder}{ox}
\icmlauthor{Jacob Bergquist}{co}
\icmlauthor{Yuan Gao}{co}
\icmlauthor{Aldo Pacchiano}{be}
\icmlauthor{Tamas Sarlos}{goo}
\icmlauthor{Adrian Weller}{ca,tu}
\icmlauthor{Vikas Sindhwani}{go}
\end{icmlauthorlist}

\icmlaffiliation{co}{Department of Industrial Engineering and Operations Research, Columbia University, New York, USA}

\icmlaffiliation{coo}{Department of Computer Science, Columbia University, New York, USA}

\icmlaffiliation{ca}{Department of Engineering, University of Cambridge, Cambridge, United Kingdom}

\icmlaffiliation{ox}{Department of Engineering, University of Oxford, Oxford, United Kingdom}

\icmlaffiliation{tu}{The Alan Turing Institute, London, United Kingdom}

\icmlaffiliation{go}{Google Brain Robotics, New York, USA}

\icmlaffiliation{goo}{Google Research, Mountain View, USA}

\icmlaffiliation{be}{Department of Computer Science, University of California, Berkeley, USA}

\icmlcorrespondingauthor{Krzysztof Choromanski}{kchoro@google.com}

\icmlkeywords{Machine Learning, ICML}

\vskip 0.3in
]



\printAffiliationsAndNotice{\icmlEqualContribution} 

\begin{abstract}
We present a new class of stochastic, geometrically-driven optimization algorithms on the orthogonal group $O(d)$ and naturally reductive homogeneous manifolds obtained from the action of the rotation group $SO(d)$. We theoretically and experimentally demonstrate that our methods can be applied in various fields of machine learning including deep, convolutional and recurrent neural networks, reinforcement learning, normalizing flows and metric learning. We show an intriguing connection between efficient stochastic optimization on the orthogonal group and graph theory (e.g. matching problem, partition functions over graphs, graph-coloring). We leverage the theory of Lie groups and provide  
theoretical results for the designed class of algorithms. We demonstrate broad applicability of our methods by showing strong performance on the seemingly unrelated tasks of learning world models to obtain  
stable policies for the most difficult $\mathrm{Humanoid}$ agent from $\mathrm{OpenAI}$ $\mathrm{Gym}$ 
and improving convolutional neural networks.
\end{abstract}

\section{Introduction and Related Work}
\label{sec:intro}

Constrained optimization is at the heart of modern machine learning (ML) \cite{bach_1, geist} as many fields of ML benefit from methods such as linear, quadratic, convex or general nonlinear constrained optimization. Other applications include reinforcement learning with safety constraints \cite{ofir}. Our focus in this paper is  \textit{on-manifold} optimization, where constraints require  solutions to belong to certain matrix manifolds. 
We are mainly interested in three spaces: the orthogonal group $\mathcal{O}(d)=\{\mathbf{M} \in \mathbb{R}^{d \times d}:\mathbf{M}^{\top}\mathbf{M}=\mathbf{I}_{d}\}$, its generalization, the \textit{Stiefel manifold} $\mathcal{ST}(d, k)=\{\mathbf{M} \in \mathbb{R}^{d \times k}:\mathbf{M}^{\top}\mathbf{M}=\mathbf{I}_{k}\}$ for $k \leq d$ 
and its subgroup $\mathcal{SO}(d)=\{\mathbf{M} \in \mathbb{R}^{d \times d}:\mathbf{M}^{\top}\mathbf{M}=\mathbf{I}_{d},\mathrm{det}(\mathbf{M})=1\}$ of $d$-dimensional rotations.

Interestingly, a wide variety of constrained optimization problems can be rewritten as optimization on matrix manifolds defined by orthogonality constraints. These include:
\begin{itemize}
\item \textbf{metric learning:} optimization leveraging the decomposition of $k$-rank Mahalanobis   matrices $\mathbf{M}=\mathbf{U}\mathbf{D}_{+}\mathbf{U}^{\top}$, where: $\mathbf{U} \in \mathcal{ST}(d,k)$ and $\mathbf{D}_{+}$ is diagonal with positive nonzero diagonal entries \cite{ankita};
\item \textbf{synchronization over the special Euclidean group}, where its elements $g_{i} \in \mathcal{SE}(d)=\mathbb{R}^{d} \times \mathcal{SO}(d)$ need to be retrieved from noisy pairwise measurements encoding the elements $g_{j}g_{i}^{-1}$ that transform $g_{i}$ to $g_{j}$; these algorithms find applications particularly in vision and robotics (SLAM) \cite{synchronization};
\item \textbf{PSD programs with block-diagonal constraints:} optimization on $\mathcal{ST}(d,k)$ can be elegantly applied for finding bounded-rank solutions to positive semidefinite (PSD) programs, where the PSD matrix $\mathbf{Z}$ can be decomposed as: $\mathbf{Z} = \mathbf{R}^{\top}\mathbf{R}$ for a low-rank matrix $\mathbf{R} \in \mathbb{R}^{r \times d}$ ($r \ll d$) and where $r$ can be systematically increased leading effectively to the so-called \textit{Riemannian staircase} method;
\item \textbf{combinatorial optimization:} combinatorial problems involving object-rearrangements such as sorting or graph isomorphism \cite{pappas} can be cast as optimization on the orthogonal group $O(d)$ which is a relaxation of the permutation group $\mathrm{Perm}(d)$.
\end{itemize}

Orthogonal groups are deeply rooted in the theory of manifolds. As noted in \cite{gallier} "\textit{ 
...most
familiar spaces are naturally reductive manifolds. Remarkably, they all arise from some suitable action of the rotation group $\mathcal{SO}(d)$, a Lie group, which emerges as the master player.}" These include in particular the Stiefel and Grassmann~manifolds.

It is striking that 
the group $O(d)$ and its ``relatives" 
provide solutions to many challenging problems in machine learning, where unstructured (i.e. unconstrained) algorithms do exist. 
Important examples include the training of recurrent and convolutional neural networks, and normalizing flows.  \textit{Sylvester normalizing flows} are used to generate complex probabilistic distributions and encode invertible transformations between density functions using matrices taken from $\mathcal{ST}(d,k)$ \cite{tomczak}. Orthogonal convolutional layers were  recently shown to improve training models for vision data \cite{ortcnn_1, BansalCW18}. Finally, the learning of deep/recurrent neural network models is notoriously difficult due to the \textit{vanishing/exploding gradient} problem \cite{frasconi}. To address this challenge, several architectures such as $\mathrm{LSTM}$s and $\mathrm{GRU}$s were proposed \cite{lstm, gru} but  
don't provide guarantees that gradients' norms would stabilize. More recently orthogonal RNNs (ORNN) that do provide such guarantees were introduced. ORNNs impose orthogonality constraints on hidden state transition matrices \cite{unitary, ortho_2, ortho_1}.
Training RNNs can now be seen as optimization on matrix manifolds with orthogonality constraints. 

However, training orthogonal RNNs is not easy and reflects challenges common for general on-manifold optimization \cite{absil, hairer, edelman}. Projection methods that work by mapping the unstructured gradient step back into the manifold have expensive, cubic time complexity. Standard geometric \textit{Riemannian-gradient} methods relying on steps conducted in the space tangent to the manifold at the point of interest and translating updates back to the manifold via \textit{exponential} or  \textit{Cayley} mapping \cite{ortho_1} still require cubic time for those transformations.
Hence, in practice orthogonal optimization becomes problematic in higher-dimensional settings.

Another problem with these transforms is that apart from computational challenges, they also incur numerical instabilities ultimately leading to solutions diverging from the manifold. We demonstrate this in Sec. \ref{sec:experiments} and provide additional evidence in the Appendix (Sec. \ref{sec:exp_prob}, \ref{app:instability}), in particular on a simple $16$-dimensional combinatorial optimization task.

An alternative approach is to parameterize orthogonal matrices by unconstrained parameters and conduct standard gradient step in the parameter space. Examples include decompositions into short-length products of Householder reflections and more \cite{rahman, jing}. Such methods enable fast updates, but inherently lack representational capacity by restricting the class of representable matrices and thus are not a subject of our work.

In this paper we propose a new class of optimization algorithms on matrix manifolds defined by orthogonality constraints, in particular: the orthogonal group $\mathcal{O}(d)$, its subgroup of rotations $\mathcal{SO}(d)$ and Stiefel manifold $\mathcal{ST}(d,k)$.

We highlight the following contributions:
\vspace{-4.0mm}
\begin{enumerate}[wide, labelwidth=!, labelindent=0pt]
    \item \textbf{Fast Optimization:} We present the first on-manifold stochastic gradient flow optimization algorithms for ML with \textbf{sub-cubic} time complexity per step as opposed to \textbf{cubic} characterizing SOTA (Sec.~\ref{time_complexity}) that do not constrain optimization to strictly lower-dimensional subspaces. 
    That enables us to improve training speed without compromising the representational capacity of the original Riemannian methods. We obtain additional computational gains (Sec.~\ref{sec:factorization}) by parallelizing our algorithms.
    \vspace{-1.5mm}
    \item \textbf{Graphs vs. Orthogonal Optimization:} To obtain these, we explore intriguing connections between stochastic optimization on the orthogonal group and graph theory (in particular the matching problem, its generalizations, partition functions over graphs and graph coloring). By leveraging structure of the Lie algebra of the orthogonal group, we map points from the rotation group to weighted graphs  (Sec. \ref{sec:algorithm}). 
    \vspace{-1.5mm}
    \item \textbf{Convergence of Stochastic Optimizers:} We provide rigorous theoretical guarantees showing that our methods converge for a wide class of functions $F:\mathcal{M} \rightarrow \mathbb{R}$ defined on $\mathcal{O}(d)$ under moderate regularity assumptions (Sec. \ref{sec:convergence}).
    \vspace{-1.5mm}
    \item \textbf{Wide Range of Applications:} We confirm our theoretical results by conducting a broad set of experiments ranging from training RL policies and RNN-based world models \cite{world_models} (Sec. \ref{sec:rl}) to improving vision CNN-models (Sec. \ref{sec:cnn}). In the former, our method is the only one that trains stable and effective policies for the high-dimensional $\mathrm{Humanoid}$ agent from $\mathrm{OpenAI}$ $\mathrm{Gym}$. We carefully quantified the impact of our algorithms. The RL experiments involved running $240$ training jobs, each distributed over hundreds of machines. We also demonstrated numerical instabilities of standard non-stochastic techniques (Sec. \ref{sec:exp_prob} and Sec. \ref{app:instability}).    
\end{enumerate}
\vspace{-3mm}

Our algorithm can be applied in particular in the blackbox setting \cite{salimans}, where function to be optimized is accessible only through potentially expensive querying process. We demonstrate it in Sec. \ref{sec:rl}, where we simultaneously train the RNN-model and RL-policy taking advantage of it through its latent states.
Full proofs of our theoretical results are in the Appendix.

\vspace{-3mm}
\section{The Geometry of the Orthogonal Group}
\label{sec:manifold}

In this section we provide the reader with technical content used throughout the paper. The theory of matrix manifolds is vast so we focus on concepts needed for the exposition of our results. We refer to \citet{lee} for a more thorough introduction to the theory of smooth manifolds.

\begin{definition}[manifold embedded in $\mathbb{R}^{n}$]
Given $n, d \in \mathbb{Z}_{\geq 1}$ with $n \geq d$, a $d$-dimensional smooth manifold in $\mathbb{R}^{n}$ is a nonempty set $\mathcal{M} \subseteq \mathbb{R}^{n}$ such that
for every $\mathbf{p} \in \mathcal{M}$ there are two open sets $\Omega \subseteq \mathbb{R}^{d}$ and $U \subseteq \mathcal{M}$ with $\mathbf{p} \in U$ and a smooth function $\phi:\Omega \rightarrow \mathbb{R}^{n}$ (called local parametrization of $\mathcal{M}$ at $\mathbf{p}$) such that $\phi$ is a homeomorphism between $\Omega$ and $U=\phi(\Omega)$, and $\mathbf{d} \phi(t_{0})$ is injective, where $t_{0}=\phi^{-1}(\mathbf{p})$.
\end{definition}

Matrix manifolds $\mathcal{M} \subseteq \mathbb{R}^{N \times N}$ are included in the above as embedded in $\mathbb{R}^{N \times N}$ after vectorization. 
Furthermore, they are usually equipped with additional natural group structure. This enables us to think about them as \textit{Lie groups} (see \citealp{lee}) that are smooth manifolds with smooth group operation. For the matrix manifolds considered here, the group operation is always standard matrix multiplication.

One of the key geometric concepts for on-manifold optimization is the notion of the tangent space.

\begin{definition} [tangent space $\mathcal{T}_{\mathbf{p}}(\mathcal{M})$]
The tangent space $\mathcal{T}_{\mathbf{p}}(\mathcal{M})$ to a smooth manifold $\mathcal{M} \subseteq \mathbb{R}^{n}$ at point $\mathbf{p} \in \mathcal{M}$ is a space of all vectors $\mathbf{v}=\gamma^{\prime}(0)$, where $\gamma:(-1,1) \rightarrow \mathcal{M}$ is a smooth curve on $\mathcal{M}$ such that $\mathbf{p}=\gamma(0)$.
\end{definition}
We refer to Sec. \ref{app:smooth} for a rigorous definition of smooth curves on manifolds. It is not hard to see that $\mathcal{T}_{\mathbf{p}}(\mathcal{M})$ is a vector space of the same dimensionality as $\mathcal{M}$ (from the injectivity of $\mathbf{d} \phi$ in $\phi^{-1}(\mathbf{p})$) that can be interpreted as a local linearization of $\mathcal{M}$. Thus it is not surprising that mapping from tangent spaces back to $\mathcal{M}$ will play an important role in on-manifold optimization (see: Subsection \ref{mappings}).

For a Lie group $\mathcal{M} \subseteq \mathbb{R}^{N \times N}$, we call the tangent space at $\mathbf{I}_{N}$ the corresponding \textit{Lie algebra}.

\subsection{Manifold $\mathcal{O}(d)$ and $\mathcal{ST}(d,k)$}

For every $\mathbf{X} \in \mathcal{ST}(d,k)$, there exists a matrix, which we call $\mathbf{X}_{\perp}$, such that  
$[\mathbf{X}\mathbf{X}_{\perp}] \in \mathcal{O}(d)$, where $[]$ stands for matrix concatenation. In fact such a matrix can be chosen in more than one way if $k < d$ (for $k=d$ we take $\mathbf{X}_{\perp}=\emptyset$).

\begin{lemma}
\label{tangent_lemma}
For every $\mathbf{X} \in \mathcal{ST}(d,k)$, the tangent space $\mathcal{T}_{\mathbf{X}}(\mathcal{M})$, where $\mathcal{M} = \mathcal{ST}(d,k)$ and $\mathrm{Sk}(k)$ stands for skew-symmetric (antisymmetric) matrices, satisfies \cite{lee}:
$$
\mathcal{T}_{\mathbf{X}}(\mathcal{M}) = 
\{\mathbf{XA}+\mathbf{X}_{\perp}\mathbf{B}: \mathbf{A} \in \mathrm{Sk}(k), \mathbf{B} \in \mathbb{R}^{(d-k) \times k}\}.
$$
\end{lemma}
\vspace{-3mm}
We conclude that tangent spaces for the orthogonal group $\mathcal{O}(d)$ are of the form: $\{\mathbf{U}\Omega: \Omega \in \mathrm{Sk}(d)\}$, where $\mathbf{U} \in \mathcal{O}(d)$
and $\mathrm{Sk}(d)$ constitutes its Lie algebra with basis $\mathcal{H}_d$ consisting of matrices $\left(\mathbf{H}_{i,j}\right)_{1 \leq i<j \leq d}$ s.t. $\mathbf{H}_{i,j}[i,j]=-\mathbf{H}_{i,j}[j,i]=1$ and $\mathbf{H}_{i,j}[k,l]=0$ if $\{k,l\} \neq \{i,j\}$.

Smooth manifolds whose tangent spaces are equipped with inner products (see: Sec. \ref{app:inner} for more details) are called \textit{Riemannian manifolds} \cite{lee}. Inner products provide a means to define a metric and talk about non-Euclidean distances between points on the manifold (or equivalently: lengths of geodesic lines connecting points on the manifold). 

\subsubsection{Exponential mapping \& Cayley transform}
\label{mappings}

Points $\mathbf{U}\Omega$ of a tangent space $\mathcal{T}_{\mathbf{U}}(\mathcal{O}(d))$ can be mapped back to $\mathcal{O}(d)$ via matrix exponentials by the mapping $\mathbf{U} \mathrm{exp}(t\Omega)$ for $t \in \mathbb{R}$.
Curves $\gamma_{\Omega}: \mathbb{R} \rightarrow \mathcal{O}(d)$ defined as $\gamma_{\Omega}(t) = \mathbf{U} \mathrm{exp}(t\Omega)$ are in fact geodesics on $\mathcal{O}(d)$ tangent to $\mathbf{U}\Omega$ in $\mathbf{U}$.
The analogue of the set of canonical axes are the geodesics induced by the canonical basis $\mathcal{H}_d$. The following lemma describes these geodesics.
\begin{lemma}[Givens rotations and geodesics]
The geodesics on $\mathcal{O}(d)$ induced by matrices $\mathbf{H}_{i,j}$ are of the form $\mathbf{U}\mathbf{G}_{i,j}^{t}$, where $\mathbf{G}_{i,j}^{t}$ is a $t$-angle Givens rotation in the $2$-dimensional space $\mathrm{Span}\{\mathbf{e}_{i},\mathbf{e}_{j}\}$  defined as:
$\mathbf{G}_{i,j}^{t}[i,i]=\mathbf{G}_{i,j}^{t}[j,j]=\cos(t), \mathbf{G}_{i,j}^{t}[i,j] = -\mathbf{G}_{i,j}^{t}[j,i]=\sin(t)$ and $\mathbf{G}_{i,j}^{t}[k,l]=\mathbf{I}_{d}[k,l]$ for $(k,l) \notin \{(i,i),(i,j),(j,i),(j,j)\}$.
\end{lemma}

All introduced geometric concepts are illustrated in Fig. \ref{fig:manifold}.

\vspace{-4mm}
\begin{figure}[h]
  \label{fig:benchmark}
  \centering
	\includegraphics[keepaspectratio, width=0.32\textwidth]{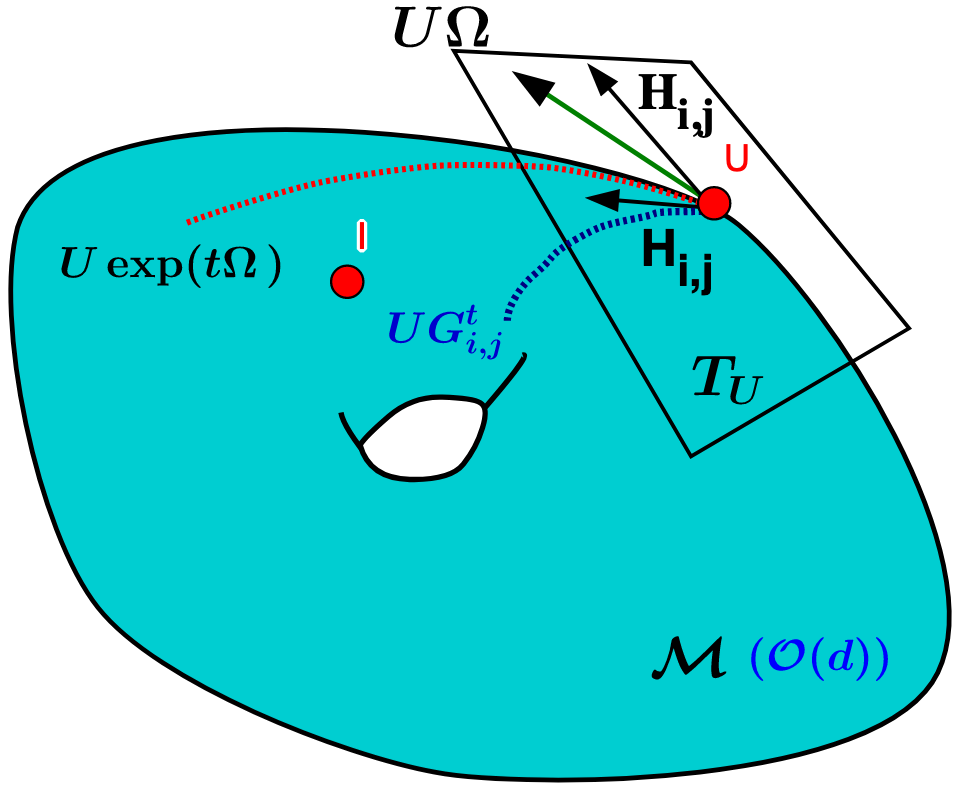}
  \caption{Illustration of basic geometric manifold concepts: tangent vector space $\mathcal{T}_{\mathbf{U}}$ at point $\mathbf{U} \in \mathcal{M}$. If $\mathcal{M}=\mathcal{O}(d)$ then 
  $\mathcal{T}_{\mathbf{U}} = \{\mathbf{U}\Omega: \Omega \in \mathrm{Sk}(d)\}$ and geodesics are of the form $\mathbf{U}\mathrm{exp}(t\Omega)$ for $t \in \mathbb{R}$. Geodesics tangent to points $\mathbf{U}\mathbf{H}_{i,j}$ are of the form $\mathbf{U}\mathbf{G}^{t}_{i,j}$.}
\label{fig:manifold}  
\end{figure}

For the general Stiefel manifold $\mathcal{ST}(d,k)$, exponentials can also be used to map skew-symmetric matrices to curves on $\mathcal{ST}(d,k)$ as follows: $\gamma_{\Omega}(t) = \mathrm{exp}(t\Omega)\mathbf{X}$, where $\mathbf{X} \in \mathcal{ST}(d,k)$ and $\mathrm{exp}$ is mapping $\mathrm{Sk}(d)$ to $\mathcal{O}(d)$. 
Analogously, curves $\gamma_{\Omega}$ are tangent to $\Omega \mathbf{X}$ in $\mathbf{X}$.
The exponential map is sometimes replaced by the \textit{Cayley transform} defined as:
$
Y(\Omega) = \left(\mathbf{I}+\frac{\Omega}{2}\right)^{-1}\left(\mathbf{I}-\frac{\Omega}{2}\right). 
$
And again, it can be shown that curves $\gamma_{\Omega}$ defined as:
$\gamma_{\Omega}(t)=Y(t\Omega)\mathbf{X}$
are tangent to $\Omega \mathbf{X}$ in $\mathbf{X}$
(since: $\gamma^{\prime}(0)=-\Omega\mathbf{X}$),
even though they are no longer geodesics.

\subsubsection{Riemannian optimization on $\mathcal{ST}(d,k)$}
\label{riemann}
\vspace{-1mm}
Consider an optimization problem of the form:
\begin{equation}
\max_{\mathbf{X} \in \mathbb{R}^{d \times k}, \mathbf{X}^{\top}\mathbf{X}=\mathbf{I}_{k}} F(\mathbf{X})
\end{equation}
\vspace{-1mm}
for a differentiable function: $F:\mathbb{R}^{d \times k} \rightarrow \mathbb{R}$.

Notice that the standard directional derivative functional $DF_{\mathbf{X}}:\mathbb{R}^{d \times k} \rightarrow \mathbb{R}$ satisfies:
$DF_{\mathbf{X}}(\mathbf{Z}) = \sum_{i,j} \frac{\partial F}{\partial \mathbf{X}[i,j]}\mathbf{Z}_{i,j}=\mathrm{tr}(\mathbf{G^{\top}\mathbf{Z}})$, where $\mathbf{G}$ is a standard gradient matrix.
To extend gradient-based techniques to on-manifold optimization, we need to:
\vspace{-2mm}
\begin{enumerate}
    \item compute the projection of $\mathbf{G}$ into the tangent space $\mathcal{T}_{\mathbf{X}}(\mathcal{ST}(d,k))$, which 
    we call \textit{Riemannian gradient},
    \vspace{-2mm}
    \item use curves on $\mathcal{ST}(d,k)$ tangent to the Riemannian gradient to make updates on the manifold.
\vspace{-2mm}
\end{enumerate}

It turns out that to make updates on the manifold, it will suffice to use the curves $\gamma_{\Omega}$ defined above. It remains to compute Riemannian gradients and corresponding skew-symmetric matrices $\Omega$.  Using standard representation theorems (see: \ref{app:representations}), it can be proven (under canonical inner product, see: Sec. \ref{app:inner}) that the Riemannian gradient is of the form $\mathcal{R} = \Omega\mathbf{X}$, where $\Omega \in \mathrm{Sk}(d)$ satisfies:
\begin{equation}
\label{omega}
\vspace{-3mm}
\Omega=\Omega(\mathbf{X},\mathbf{G}) = \mathbf{GX}^{\top} - \mathbf{XG}^{\top}.  
\end{equation}

We are ready to conclude that for $\mathbf{X}(t) \in \mathcal{ST}(d,k)$, a differential equation (DE):
\begin{equation}
\label{diff}
\mathbf{\dot{X}}(t)= \Omega(t) \mathbf{X}(t),
\end{equation}
where $\Omega(t) = \mathbf{G}(t)\mathbf{X}(t)^{\top} - \mathbf{X}(t)\mathbf{G}(t)^{\top}$ and $\mathbf{G}(t)$ is a standard gradient matrix (for unconstrained optimization) in $\mathbf{X}(t)$ for a given function $F:\mathcal{ST}(d,k) \rightarrow \mathbb{R}$ to be maximized, 
defines a gradient-flow on $\mathcal{ST}(d,k)$.
The flow can be discretized and integrated with the use of either the exponential map or the Cayley transform $Y$ as follows:
\begin{equation}
\label{integrate}
\mathbf{X}_{i+1} = \Gamma(\eta \Omega(\mathbf{X}_{i},\mathbf{G}_{i}))\mathbf{X}_{i},    
\end{equation}
where $\mathbf{G}_{i}$ is a standard gradient matrix in $\mathbf{X}_{i}$, $\Gamma = Y$ or $\Gamma = \mathrm{exp}$ and $\eta>0$ is a step size. 

Eq. \ref{integrate} leads to an optimization algorithm with \textbf{cubic} time per optimization step (both exponential and Cayley transforms require cubic time), prohibitive in higher-dimensional 
applications. We propose our solution in the next section.

Interestingly, if $\mathcal{ST}(d,k)=\mathcal{O}(d)$ (i.e. $d=k$), the order of terms in the RHS of Equation \ref{diff} can be reversed and such
DE also defines a gradient flow for $\Omega(\mathbf{X},\mathbf{G}) = \mathbf{X}^{\top}\mathbf{G}-\mathbf{G}^{\top}\mathbf{X}$ that can be integrated via Equation \ref{integrate}, with the order of terms in RHS reversed. We refer to it as a \textit{reverse form}.

\vspace{-3mm}
\section{Graph-Driven On-Manifold Optimization}
\label{sec:algorithm}

Our algorithms provide an efficient way of computing discretized flows given by Eq. \ref{diff}, by proposing graph-based techniques for time-efficient stochastic approximations of updates from Eq. \ref{integrate}. We focus on $\Gamma$ being an exponential map since it has several advantages over Cayley transform (in particular, it is surjective onto the space of all rotations).

We aim to replace $\Omega(\mathbf{X}_{i},\mathbf{G}_{i})$ with an unbiased sparser low-variance stochastic structured estimate $\widehat{\Omega}$. This will enable us to replace {cubic} time complexity of an update with {sub-cubic}. Importantly, an unbiased estimator of $\Omega$ does not lead to an unbiased estimator of $\mathbf{X}_{i+1}$. However, as seen from Taylor expansion, the bias can be controlled by the step size $\eta$ and in practice will not hurt accuracy. We verify this claim both  theoretically (Sec. \ref{sec:convergence}) by providing {strong convergence results}, and empirically (Sec. \ref{sec:experiments}).

A key observation linking our setting with graph theory is that elements of the Lie algebra of $\mathcal{O}(d)$ can be interpreted as adjacency matrices of \textit{weighted tournaments}. Then, specific subsamplings of subtournaments lead to efficient updates. To explain this we need the following definitions.

\begin{definition}[weighted tournament]
For $\mathbf{W} \in \mathrm{Sk}(d)$, the weighted tournament $T(\mathbf{W})$ is a directed graph with vertices $\{0,1,...,d-1\}$ and edges $E(\mathbf{W})$, where $(i,j) \in E(\mathbf{W})$ iff $\mathbf{W}[i,j] > 0$. $\mathbf{W}$ is called an adjacency matrix of $T(\mathbf{W})$. The induced undirected graph is denoted $G_{T(\mathbf{W})}$ and its set of edges $E(G_{T(\mathbf{W})})$. A subtournament $S$ of $T(\mathbf{W})$ is obtained by deleting selected edges of $T(\mathbf{W})$ and zeroing corresponding entries in the adjacency matrix $\mathbf{W}$, to form an adjacency matrix of $S$ (denoted $\mathbf{W}[S]$). 
\end{definition}

\vspace{-2mm}
Assume we can rewrite $\Omega \in \mathbb{R}^{d \times d}$ as:
\begin{equation}
\label{approx}
\Omega = \sum_{T \in \mathcal{T}} p_{T} \Omega_{T},
\end{equation}
where: $\mathcal{T}$ is a family of subtournaments of $T(\Omega)$, $\Omega_{T} \in \mathrm{Sk}(d)$ and $\{p_{T}\}_{T \in \mathcal{T}}$ is a discrete probability distribution on $\mathcal{T}$. Then Eq. \ref{approx} leads to the unbiased estimator $\widehat{\Omega}$ of $\Omega$, when $\widehat{\Omega} = \Omega_{T}$ with probability $p_{T}$. Such a decomposition is possible only if each edge in $E(\Omega)$ appears at least in one $T$. We call this a \textit{covering family}.

For a covering family $\mathcal{T}$, the matrices $\Omega_{T}$ can be defined as:
\begin{equation}
    \Omega_{T}=\frac{1}{p_{T}}\mathbf{M}_{\mathcal{T}} \odot \Omega[T] \label{eqn:estimator}
\end{equation} where $\mathbf{M}_{\mathcal{T}}[i,j] >0$ is the inverse of the number of $T \in \mathcal{T}$ such that $\{i,j\} \in E(G_{T})$ and $\odot$ stands for the Hadamard product.
We propose to replace the update from Eq. \ref{integrate} with the following stochastic variant (with probability $p_{T}$):
\begin{equation}
\label{stoch_update_basic}
\mathbf{X}_{i+1} = \Gamma(\eta \Omega_{T}(\mathbf{X}_{i},\mathbf{G}_{i}))\mathbf{X}_{i}    
\end{equation}
We focus on the family of unbiased estimators $\widehat{\Omega}$ given by equations \ref{approx} and \ref{eqn:estimator}. In order to efficiently apply the above stochastic approach for on-manifold optimization, we need to construct a family of subtournaments $\mathcal{T}$ and a corresponding probabilistic distribution $\{p_{T}\}_{T \in \mathcal{T}}$ such that: \textbf{(1)} update of $\mathbf{X}_i$ given $\Omega_{T}(\mathbf{X}_i,\mathbf{G}_i)$ can be computed in sub-cubic time, \textbf{(2)} sampling $\Omega_{T}$ via $\mathcal{P} = \{p_{T}\}$ can be conducted in sub-cubic time, \textbf{(3)} the stochastic update is accurate enough. 

\vspace{-3mm}
\subsection{Multi-connected-components graphs}
\label{sec:factorization}

To achieve \textbf{(1)}, we consider a structured family of subtournaments $\mathcal{T}$. For each $T \in \mathcal{T}$, the corresponding undirected graph $G_{T}$ consists of several connected components. Then the RHS of Eq. \ref{stoch_update_basic} can be factorized leading to an update:
\vspace{-3mm}
\begin{equation}
\label{stoch_update}
\mathbf{X}_{i+1} = \left ( \prod_{k=1,...,l}\Gamma(\eta \Omega_{{T}_{k}}(\mathbf{X}_{i},\mathbf{G}_{i}))  \right ) \mathbf{X}_{i},    
\vspace{-4mm}
\end{equation}

where $l$ stands for the number of connected components and $T_{k}$ is a subtournament of $T$ such that $G_{T_{k}}$ is $G_T$'s $k^{th}$ connected component. We can factorize because matrices $\{\Omega_{T_{k}}\}_{k=1,...,l}$ all commute since they correspond to different connected components.
We will construct the connected components to have the same number of vertices $s = \frac{d}{l}$ (w.l.o.g we can assume that $l ~|~d$) such that the corresponding matrices $\Omega_{T_k}$ have at most $s$ non-zero columns/rows. 

Time complexity of right-multiplying $\Gamma(\eta \Omega_{{T}_{k}}(\mathbf{X}_{i},\mathbf{G}_{i}))$ by a $d \times d$ matrix is $O(ds^{2})$ thus total time complexity of the update from Eq. \ref{stoch_update} is $T_{\mathrm{update}}=O(ds^{2}l)=O(d^{2}s)$.
Thus if $s=o(d)$, we have: $T_{\mathrm{update}}=o(d^{3})$. In fact the update from Eq. \ref{stoch_update} can be further  GPU-parallelized across $l$ threads (since different matrices $\Gamma(\eta \Omega_{{T}_{k}}(\mathbf{X}_{i},\mathbf{G}_{i}))$ modify disjoint subsets of entries of each column of the matrix that $\Gamma(\eta \Omega_{{T}_{k}}(\mathbf{X}_{i},\mathbf{G}_{i}))$ is right-multiplied by) leading to total time complexity $T_{\mathrm{update}}^{\mathrm{parallel}}=O(ds^{2})$ after GPU-parallelization.

Let $\mathcal{T}_{s}$ be the combinatorial space of all subtournaments $T$ of $T(\Omega)$ such that $G_{T}$ has $s$-size connected components. Note that for $2 \leq s < d$, $\mathcal{T}_{s}$ is of size exponential in $d$.

\vspace{-3mm}
\subsection{Sampling subtournaments}
\label{sec:sampling_subtournaments}

Points \textbf{(2)} and \textbf{(3)} are related to each other thus we address them simultaneously. We observe that sampling uniformly at random $\mathcal{T} \sim \mathrm{Unif}(\mathcal{T}_s)$ is trivial: we randomly permute vertices $\{0,1,...,d-1\}$ (this can be done in linear time, see: Python Fisher-Yates shuffle), then take first $s$ to form first connected component, next $s$ to form next one, etc. For the uniform distribution over $\mathcal{T}_{s}$ we have: $\frac{1}{p_T}\mathbf{M}_{\mathcal{T}_s} = \frac{d-1}{s-1}\mathbf{J}_{d}$, where $\textbf{J}_{d} \in \mathbb{R}^{d \times d}$ is all-one matrix (see Lemma \ref{lemma:uniform} in the Appendix), thus we have:
$\Omega_{T} = \frac{d-1}{s-1} G_T$.
We conclude that sampling $\Omega_{T}$ can be done in time
$O(\frac{d}{s}{s \choose 2})=O(ds)$. 

Instead of $\mathcal{T}_{s}$, one can consider a much smaller family $\mathcal{T}$, where graphs $G_{T}$ for different $T \in \mathcal{T}$ consist of disjoint sets of edges (and every edge belongs to one $G_T$). We call such $\mathcal{T}$ \textit{non-intersecting}. It is not hard to see that in that setting $\mathbf{M}_{\mathcal{T}} = \mathbf{J}_{d}$ (since edges are not shared across different graphs $G_T$). Also, the size of $\mathcal{T}$ is at most quadratic in $d$. Thus sampling $\Omega_{T}$ can be conducted in time $O(d^{2})$.

In both cases, $\mathcal{T}=\mathcal{T}_s$ and $\mathcal{T}$ \textit{non-intersecting}, we notice that $\mathbf{M}_\mathcal{T} \propto \mathbf{J_d}$. We call such a $\mathcal{T}$ a \textit{homogeneous family}. 

\subsubsection{Non-uniform distributions}
\vspace{-1mm}

To reduce variance of $\widehat{\Omega}$, uniform distributions should be replaced by  non-uniform ones (detailed analysis of the variance of different methods proposed in this paper is given in the Appendix (Sec. \ref{app:var}, \ref{app:extra_combinatorica})). Intuitively, matrices $\Omega_T$ for which graphs $G_T$ have large absolute edge-weights should be prioritized and given larger probabilities $p_T$. 
We denote by $\|\|_{\mathcal{F}}$ the Frobenius norm and by $w_e$ the weight of edge $e$.

\begin{lemma}[importance sampling]
\label{importance_sampling}
Given $\mathcal{T}$, a distribution $\mathcal{P}_{\mathcal{T}}^{\mathrm{opt}}$ over $\mathcal{T}$ producing unbiased $\widehat{\Omega}$ and minimizing variance $\mathrm{Var}(\widehat{\Omega})=\mathbb{E}
[\|\widehat{\Omega} - \Omega \|^{2}_{\mathcal{F}}]$ satisfies: 
$p^{\mathrm{opt}}_{T} \sim \sqrt{\sum_{(i,j) \in E(G_T)} \left(M_{\mathcal{T}} \odot  \Omega\right)[i,j]^2 }$. For homogeneous families, it simplifies to $p^{\mathrm{opt}}_{T} \sim \sqrt{\sum_{e \in E(G_T)} w_e^2 }$.
\end{lemma}
\vspace{-1mm}
From now on, we consider homogeneous families.
Sampling from optimal $\mathcal{P}$ is straightforward if $\mathcal{T}$ is non-intersecting and can be conducted in time $O(d^{2})$, but becomes problematic for families $\mathcal{T}$ of exponential sizes.
We now introduce a rich family of distributions for which sampling can be conducted efficiently for homogeneous $\mathcal{T}$.

\begin{definition}[$h$-regular distributions]
\label{def:hregular}
For even function $h:\mathbb{R} \rightarrow \mathbb{R}_{+}$ such that $h(0)=0$,
we say that distribution $\mathcal{P}^{h}(\mathcal{T})$ over $\mathcal{T}$ is $h$-regular if 
$p^{h}_{T} \sim \sum_{e \in E(G_T)} h(w_e)$.
\end{definition}

The core idea is to sample uniformly at random, which we can do for $\mathcal{T}_s$, but then accept sampled $T$ with certain easily-computable probability $q_T^h$. If $T$ is not accepted, the procedure is repeated. Probability $q_T^h$ should satisfy $q_T^h = \lambda p^{h}_{T}$, for a renormalization term $\lambda>0$.

\vspace{-2mm}  
\begin{algorithm}[H]
\caption{Constructing tournament $T \sim \mathcal{P}^{h}(\mathcal{T})$ and $\Omega_T$}
\textbf{Hyperparameters:} $0 < \alpha, \beta < 1$ (only in version II)\; \\
\textbf{Input:}  $\Omega \in \mathrm{Sk}(d)$\; \\
\textbf{Output:} subtournament $T$ and corresponding $\Omega_T$ \; \\
\textbf{Preprocessing:} Compute $\tau=\rho \|h(\Omega)\|_{1}$,
where $\rho=1$ (version I) or $\rho=\frac{s}{d \alpha \beta}$ (version II). \; \\
\While{$\mathrm{True}$}{
  1. sample $T \sim \mathrm{Unif}(\mathcal{T}_s)$ (see: Sec. \ref{sec:sampling_subtournaments}),\; \\
  2. compute $q_T^h= \frac{2h(G_T)}{\tau}$
  ($h(G_T)$ as in Lemma \ref{probability_formula}), \; \\
  3. with probability $q_T^h$ return $\left(T,\frac{\|h(\Omega)\|_1}{2h(G_T)}G_T\right)$.
}
\label{Alg:asebo}
\end{algorithm}
\vspace{-6mm}

The larger $\lambda$, the smaller the expected number of trials needed to accept sampled $T$. Indeed, the following is true:
\begin{lemma}
\label{expected_number_of_trials}
The expected number of trials before sampled $T$ is accepted is: $\frac{|\mathcal{T}_{s}|}{\lambda}$, where $|\mathcal{X}|$ stands for the size of $\mathcal{X}$. 
\end{lemma}

On the other hand, we must have: $\lambda \leq \frac{1}{\max _T p _{T}^{h}}$. The following result enables us to choose large enough $\lambda > 0$, so that the expected number of trials remains small.

\begin{lemma}
\label{probability_formula}
For an $h$-regular distribution on $\mathcal{T}_s$,  probability $p_{T}^{h}$ is given as: $ p^{h}_{T}=\frac{2h(G_T)}{W \|h(\Omega)\|_{1}} \leq \frac{1}{W}$, where $h(G_T) = \sum_{e \in E(G_T)} h(w_e)$, $h(\Omega)=[h(\Omega[i,j])]_{i,j \in \{0,1,...,d-1\}}$, 
$\|h(\Omega)\|_{1}=\sum_{i,j}h(\Omega[i,j])$
and
$W=(d-2)!/\left((s-2)!(s!)^{\frac{d-s}{s}}(\frac{d-s}{s})!\right)$.
\end{lemma}

Lemma \ref{probability_formula} enables us to choose $\lambda=W$ and consequently: $q^h_T = \frac{2h(G_T)}{\|h(\Omega)\|_{1}}$. 
We can try to do even better, by taking: 
$\lambda = \frac{W\|h(\Omega)\|_{1}}{\tau}$
for any $\tau$ such that $\tau^{*} = 2\max_{T \in \mathcal{T}_{s}} h(G_T)  \leq \tau < \|h(\Omega)\|_{1}$, leading to: 
$q^h_T = \frac{2h(G_T)}{\tau}$. 

Efficiently finding a nontrivial (i.e. smaller than $\|h(\Omega)\|_{1}$) upper bounds $\tau$ on $\tau^{*}$ 
is not always possible, but can be trivially done for $(\alpha,\beta, h)$-balanced matrices $\Omega$, i.e. $\Omega$ such that at least an $\alpha$-fraction of all entries $\Omega[i,j]$ of $\Omega$
satisfy: $h(\Omega[i,j]) \geq \beta \max_{a,b} h(\Omega[a,b])$.
It is not hard to see that for such $\Omega$ one can take:
$\tau = O(\frac{s}{d \alpha \beta}\|h(\Omega)\|_{1})$ or $\tau=O(\frac{s}{d}\|h(\Omega)\|_{1})$ for $\alpha^{-1},\beta^{-1}=O(1)$. 
We observed (see: Section \ref{app:ablation}) that in practice one can often take $\tau$ of that order.
Our general method for sampling $T \in \mathcal{T}_s$ and the corresponding $\Omega_T$ is given in $\mathrm{Alg.}$ 1.
We have:
\begin{theorem}
\label{iter-time-complexity}
Alg. 1 outputs $T \in \mathcal{T}_s$ from distribution $\mathcal{P}^{h}(\mathcal{T}_s)$ and corresponding $\Omega_T$. Its expected 
time complexity is $O(\frac{d^2 \gamma \tau}{\|h(\Omega)\|_{1}})+\xi$, where $\gamma$ is time complexity for computing a fixed entry of $\Omega$ and $\xi$ - for computing $\|h(\Omega)\|_1$.
\end{theorem}

\vspace{-4mm}
\subsubsection{Extensions}
\paragraph{Dynamic domains $\mathcal{T}$:} One can consider changing sampling domains $\mathcal{T}$ across iterations of the optimization algorithm. If non-intersecting families are used (see: Sec. \ref{sec:sampling_subtournaments}), $\mathcal{T}$ can be chosen at each step (or periodically) to minimize the optimal variance given by Lemma \ref{importance_sampling}. Optimizing $\mathcal{T}$ is almost always a nontrivial combinatorial problem.
We shed light on these additional intrinsic connections between combinatorics and on-manifold optimization in Sec. \ref{app:extra_combinatorica}.

\vspace{-5mm}
\paragraph{Partition functions:} Our sampling scheme can be modified to apply to distributions from
Lemma \ref{importance_sampling}. The difficulty lies in obtaining a renormalization factor for probabilities $p^{h}_{T} \sim \sqrt{\sum_{e \in E(G_T)} w_{e}^{2}}$ which is a nontrivial graph partition function. However it can  often be approximated by Monte Carlo methods \cite{jain}. In practice we do not need to use variance-optimal optimizers to obtain good results. 

\vspace{-3mm}
\subsection{Time Complexity of the stochastic algorithm}
\label{time_complexity}

To summarize, we can conduct single optimization step given a sampled $T$ and $\Omega_T$ in time $O(d^{2}s)$ or even $O(ds^{2})$ after further GPU-parallelization.
If $\Omega$ is given then, due to Theorem \ref{iter-time-complexity} and by previous analysis, entire sampling procedure can be done in $O(d^{2})$ time leading to sub-cubic complexity of the entire optimization step. If $\Omega$ is not given, then notice first that computing a fixed entry of $\Omega$ takes time $O(k)$ for $\mathcal{ST}(d,k)$ and $O(d)$ for $\mathcal{O}(d)$ or $\mathcal{SO}(d)$. If sampling is conducted uniformly, then it invokes getting ${s \choose 2}\frac{d}{s}=O(ds)$ edges of the graph (i.e. entries of $\Omega$) and thus total complexity is still sub-cubic.
Furthermore, for $\mathcal{ST}(d,k)$ and $k=o(d)$ that remains true even for non-uniform sampling due to Theorem \ref{iter-time-complexity} since Alg. 1 runs in $O(d^{2}k)$ as opposed to $O(d^{3})$ time. 

Finally, for $\mathcal{O}(d)$, $\mathcal{SO}(d)$ and when $d=O(k)$, by Theorem \ref{iter-time-complexity}, non-uniform sampling can be conducted in time 
$O(\frac{d^{2}s}{\alpha \beta}) + \xi$ for $(\alpha, \beta, h)$-balanced $\Omega$. The first term is sub-cubic when $\alpha^{-1}\beta^{-1} = o(\frac{d}{s})$ as we confirm in practical applications (see: Sec. \ref{app:ablation} and our discussion above).
In practice for such $\Omega$, the value $\|h(\Omega)\|_{1}$ can be accurately estimated for arbitrarily small relative error $\epsilon>0$ with probability $p=1-O(\mathrm{exp}(-(\frac{\epsilon \alpha \beta r}{3d})^{2}))$ by simple Monte Carlo procedure that approximates $h(\Omega)$ with its sub-sampled version of only $O(r)$ nonzero entries (see: Sec. \ref{app:mc}). In practice one can choose $r=o(d^{2})$ resulting in $\xi=o(d^{3})$ and sub-cubic complexity of the entire step.

\vspace{-3mm}
\section{Optimizing with Graph Matchings} \label{matchings}
Note that if $\mathcal{T}$ is chosen in such a way that every connected component of $G_T$ consists of two vertices then $G_T$ is simply a \textit{matching} \cite{diestel}, i.e. a graph, where every vertex is connected to at most one more vertex. In particular, if $s=2$, then $\{G_T:T \in \mathcal{T}_s\}$ is a collection of all \textit{perfect matchings} of $G_{T(\Omega)}$, i.e. matchings, where every vertex belongs to some edge. For such $\mathcal{T}$s, the update rule given by Eq. \ref{stoch_update} has particularly elegant form, namely:
\vspace{-1mm}
\begin{equation}
\mathbf{X}_{i+1}=\prod_{k=1,...,l}\mathbf{G}^{\theta_k}_{i_k,j_k} \mathbf{X}_{i},  
\vspace{-2mm}
\end{equation}
for some $\theta_{1},...,\theta_{l} \in [0, 2\pi]$ and
$i_1,j_1,...,i_k,j_k$.
In other words, the update is encoded by the product of Givens rotations (see: Section \ref{mappings}).
This sheds new light on recently proposed algorithms using products of Givens rotations to learn neural networks for RL \cite{kamas} or to approximate Haar measure on $\mathcal{O}(d)$, (see: Sec. \ref{app:givens}).

Constructing a good non-intersecting family $\mathcal{T}$ in this setting is intrinsically related to graph matching/edge coloring optimization problems since $\mathcal{T}$ can be obtained by iteratively finding heavy matchings (i.e. monochromatic classes of valid edge colorings) of $G_T$.
In Sec. \ref{app:extra_combinatorica} we explain these connections in more detail. Even though not central for our main argument, they provide additional deeper context.

\vspace{-3mm}
\section{Convergence Results}
\label{sec:convergence}

We show that our stochastic optimizers have similar convergence rates as deterministic ones. The difference is quantified by the term $\sigma$ related to the variance of the estimator of $\Omega$ (see below). Without loss of generality we consider optimization on $\mathcal{O}(d)$. Analogous results hold for $\mathcal{ST}(d,k)$.

\begin{theorem} 
\label{osgd}
Let $F: \mathbb{R}^{d \times d} \to \mathbb{R}$ be such that  standard gradient $\nabla F$ is defined on $\mathcal{O}(d)$ and for all $ \mathbf{M}, \mathbf{N} \in \mathcal{O}(d),$
\begin{equation} \label{lipschitz}
     \| \nabla F (\mathbf{M}) - \nabla F (\mathbf{N}) \|_\mathcal{F} \leq L \| \mathbf{M} - \mathbf{N} \|_\mathcal{F},
\end{equation}
where $L>0$ \footnote{As noted by \citet{shalit}, because $\mathcal{O}(d)$ is compact, any $F$ with a continuous second derivative will obey (\ref{lipschitz}).} and $\| \cdot \|_\mathcal{F}$ is a Frobenius norm.  Let $\{ \mathbf{X}_i \}_{i \geq 0}$ be the sequence generated by the proposed stochastic update. $\mathbf{X}_{0} \in \mathcal{O}(d)$ is fixed and $\mathbf{X}_{i + 1} := \exp (\eta_i \widehat{\Omega}_i) \mathbf{X}_i$, where $\widehat{\Omega}_i$ is drawn from distribution $\mathbb{P} (\widehat{\Omega}_i)$  defined on $\mathrm{Sk}(d)$  s.t. $\mathbb{E} \widehat{\Omega}_i = \Omega_i$ and $\Omega_i := \nabla F (\mathbf{X}_i) \mathbf{X}_i^\top - \mathbf{X}_i \nabla F (\mathbf{X}_i)^\top$.
Here $\{ \eta_i > 0 \}_{i \geq 0}$ is a sequence of step sizes. Then
\vspace{-3mm}
\begin{align*}
    \min_{i = \overline{0..T}} & \mathbb{E} \| \nabla_\mathcal{O} F (\mathbf{X}_i) \|^2_\mathcal{F} \leq 2 \frac{F^* - F(\mathbf{X}_0)}{\sum_{i = 0}^T \eta_i} + \Sigma_T 
\end{align*}
where $\Sigma_T=\sigma^2 \biggl( (2 \sqrt{d} + 1) L + \| \nabla F (\mathbf{I}_{d}) \|_\mathcal{F} \biggr) \frac{\sum_{i = 0}^T \eta^2_i}{\sum_{i = 0}^T \eta_i}$,
$\nabla_{\mathcal{O}}F$ denotes a Riemannian gradient (see: Sec. \ref{riemann}), $F^* = \sup_{\mathbf{X} \in \mathcal{O}(d)} F(\mathbf{X})$ and $\sigma^2 > 0$ is chosen so that $\forall \Omega \in \{ \nabla_\mathcal{O} F (\mathbf{X}) \mathbf{X}^\top | \mathbf{X} \in \mathcal{O}(d) \} : \sigma^2 \geq \mathbb{E} \| \widehat{\Omega} \|^2_\mathcal{F}$.
\end{theorem}

\section{Experiments}
\label{sec:experiments}

\subsection{RL with Evolution Strategies and RNNs}
\label{sec:rl}
Here we demonstrate the effectiveness of our approach in optimizing policies for a variety of continuous RL tasks from the $\mathrm{OpenAI}$ $\mathrm{Gym}$ ($\mathrm{Humanoid}$, $\mathrm{Walker2d}$, $\mathrm{HalfCheetah}$) and $\mathrm{DM}$ $\mathrm{Control}$ $\mathrm{Suite}$ ($\mathrm{Reacher:Hard}$, $\mathrm{HopperStand}$ and $\mathrm{Swimmer:15}$). 

\vspace{-4mm}
\begin{figure}[H]
    \begin{minipage}{0.49\textwidth}
    \subfigure{\includegraphics[width=.99\linewidth]{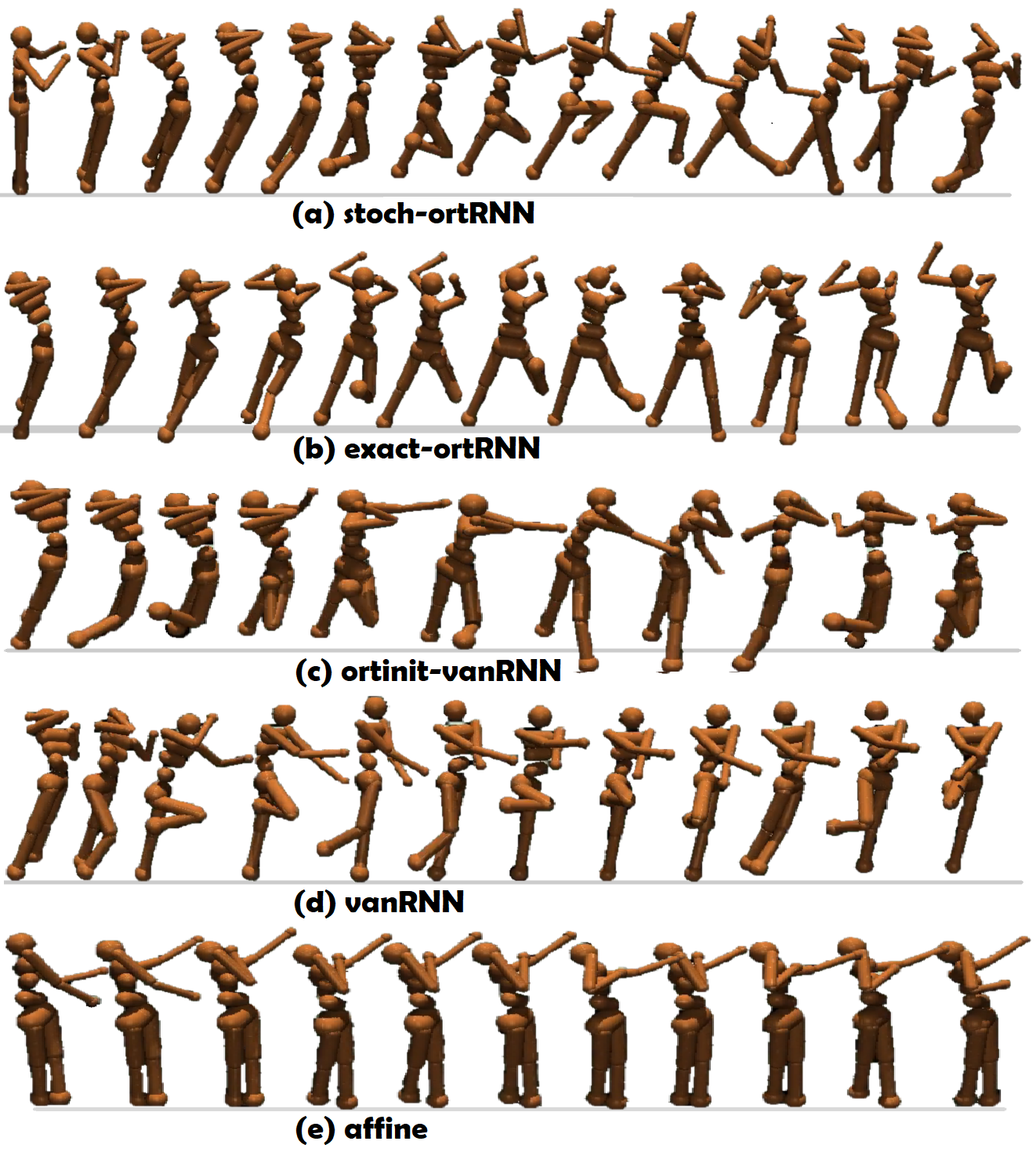}}
    \vspace{-3mm}
    \end{minipage}
    \caption{\small{Visualizations of policies learned by different algorithms for $\mathrm{Humanoid}$ from $\mathrm{OpenAI}$ $\mathrm{Gym}$.}}
    \label{figure:humanoid}   
\end{figure}
\vspace{-5mm}

We aim at jointly learn the RNN-based world model \cite{world_models} and a policy affine in the latent (but not original) state, using evolution strategy optimization methods, recently proven to match or outperform SOTA policy gradient algorithms \cite{salimans}. We conduct stochastic optimization on $\mathcal{O}(d)$ by constraining transition-matrices of RNNs to be orthogonal (see: Sec. \ref{sec:intro}). 

\vspace{-4mm}
\paragraph{Compared methods:} Our algorithm 
($\mathrm{stoch}$-$\mathrm{ortRNN}$) applying matching-based sampling with $h:x \rightarrow |x|$
is compared with three other methods: 
\textbf{(1)} its deterministic variant where exponentials are explicitly computed ($\mathrm{exact}$-$\mathrm{ortRNN}$),
\textbf{(2)} unstructured vanilla RNNs ($\mathrm{vanRNN}$), \textbf{(3)} vanilla RNN with orthogonal initialization of the transition matrix and periodic projections back into orthogonal group ($\mathrm{ortinit}$-$\mathrm{vanRNN}$).
For the most challenging Humanoid task we also trained purely affine policies.

Comparing with \textbf{(3)} enables us to measure incremental gain coming from the orthogonal optimization as opposed to just orthogonal initialization which was proven to increase performance of neural networks \cite{orthogonal_init}. We observed that just orthogonal initialization does not work well (is comparable to vanilla RNN) thus in $\mathrm{ortinit}$-$\mathrm{vanRNN}$ we also periodically every $p$ iterations project back onto orthogonal group. We tested different $p$ and observed that in practice it needs to satisfy $p \leq 20$ to provide training improvements. We present variant with $p=20$ in Fig. \ref{figure:es_rnn_openai} since it is the fastest. Detailed ablation studies with different values of $p$ are given in Table 1.

\vspace{-4mm}
\begin{figure}[H]
    \begin{minipage}{0.49\textwidth}
    \subfigure[\textbf{Reacher:Hard}]{\includegraphics[width=.49\linewidth]{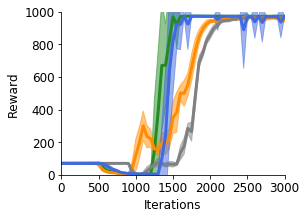}}
    \vspace{-2mm}
    \subfigure[\textbf{Hopper: Stand}]{\includegraphics[width=.46\linewidth]{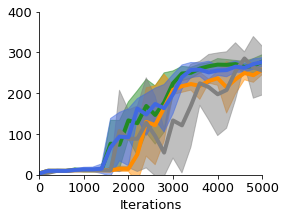}}  
    \end{minipage}
    \begin{minipage}{0.49\textwidth}
    \subfigure[\textbf{Swimmer:15}]{\includegraphics[width=.49\linewidth]{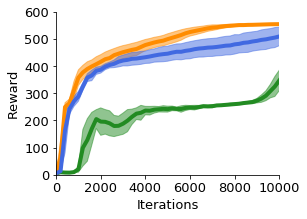}}  
    \subfigure[\textbf{HalfCheetah}]{\includegraphics[width=.48\linewidth]{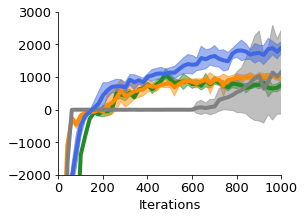}}  
    \vspace{-0.5mm}
    \end{minipage}    
\end{figure}
\vspace{-6mm}
\begin{figure}[H]
\vspace{-1.3mm}
    \vspace{-3mm}  
    \begin{minipage}{0.49\textwidth}
    \subfigure[\textbf{Walker2d}]{\includegraphics[width=.50\linewidth]{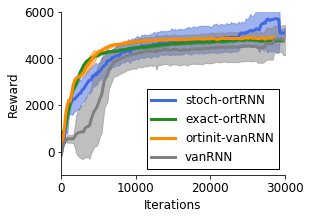}}  
    \subfigure[\textbf{Humanoid}]{\includegraphics[width=.46\linewidth]{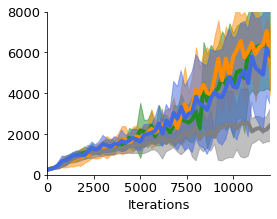}}  
    \vspace{-3mm}
    \caption{\small{Comparison of all RNN-based algorithms on DM Control Suite (a-c) and OpenAI Gym (d-f) tasks: Each plot shows the mean +- one $\mathrm{stdev}$ across $s=10$ seeds. For $\mathrm{Swimmer:15}$ we present only three curves since $\mathrm{vanRNN}$ did not train at all.}}
    \label{figure:es_rnn_openai}
    \end{minipage}        
\end{figure}
\vspace{-5mm}

Experiment with affine policies helps us to illustrate that training nontrivial hidden state dynamics is indeed crucial. Purely affine policies were recently showed to provide high rewards for $\mathrm{OpenAI}$ environments
 \cite{mania}, but we demonstrate that they lead to inferior agent's behaviors, thus their rewards are deceptive. To do that, we simulated all learned policies for the most challenging $376$-dimensional $\mathrm{Humanoid}$ environment (see video library in Appendix).

\vspace{-3.5mm}
\paragraph{Setting:} For each method we run optimization for $s=10$ random seeds (resulting in $240$ experiments)  and hidden state of sizes $h=200,400$ (similar results in both settings). 
For all environments distributed optimization is conducted on $800$ machines.
\vspace{-3mm}
\paragraph{Results:}
In Fig. \ref{figure:humanoid} we show most common policies learned by all four RNN-based algorithms and trained affine policy for the most challenging $\mathrm{Humanoid}$ task. Our method was the only one that consistently produced good walking/running behaviors. Surprisingly, most popular policies produced by $\mathrm{exact}$-$\mathrm{ortRNN}$ while still effective, were clearly less efficient. 
The behavior further deteriorated if orthogonal optimization was replaced by $\mathrm{ortinit}$-$\mathrm{vanRNN}$. Pure vanilla RNNs produced unnatural movements, where agent started spinning and jumping on one leg. For affine policies forward movement was almost nonexistent. 
We do believe we are the first to apply and show gains offered by orthogonal RNN architectures for RL.

We conjecture that the superiority of our method over $\mathrm{exact}$-$\mathrm{ortRNN}$ is partially due to numerical instabilities of standard Riemannian optimization on $\mathcal{O}(d)$ (see: Appendix, Sec \ref{sec:exp_prob}). In the Appendix (Sec. \ref{app:instability}) we demonstrate that it is the case even for very low-dimensional tasks ($d=16$).

In Fig. \ref{figure:es_rnn_openai} we present corresponding training curves for all analyzed RNN-based methods. For all environments $\mathrm{stoch}$-$\mathrm{ortRNN}$ method did well, in particular in comparison to $\mathrm{exact}$-$\mathrm{ortRNN}$.
In Sec. \ref{exp:tc} we show that $\mathrm{stoc}$-$\mathrm{ortRNN}$ is much faster than other methods using ortho-constraints.

\vspace{-3mm}
\subsection{Optimization on $\mathcal{ST}(d,k)$ for Vision-Based Tasks}
\label{sec:cnn}

Here we apply our methods to orthogonal convolutional neural networks \cite{Jia19,HuangLLYWL18, XieXP17, BansalCW18, ortcnn_1}.

\vspace{-3.9mm}
\paragraph{Setting:} For MNIST, we used a 2-layer MLP with each layer of width 100, and $\mathrm{tanh}$ activations specifically to provide a simple example of the vanishing gradient problem, and to benchmark our stochastic orthogonal integrator. The matching-based variants of our algorithm with uniform sampling was accurate enough. For CIFAR10, we used a PlainNet-110 (ResNet-110 \cite{HeZRS16} without residual layers), similar to the experimental setting in \cite{XieXP17}. For a convolutional kernel of shape $[H, W, C_{in}, C_{out}]$ (denoting respectively $[$\texttt{height, width, in-channel, out-channel}$]$), we impose orthogonality on the flattened 2-D matrix of shape $[H * W * C_{in}, C_{out}]$ as commonly used in \cite{Jia19,HuangLLYWL18, XieXP17, BansalCW18, ortcnn_1}. We 
applied our algorithm for $s \geq 2$ and uniform sampling.
Further training details are in the Appendix \ref{appendix_SL}.

\vspace{-3mm}
\paragraph{Results:}In Fig. \ref{mnist}, we find that the stochastic optimizer (when $s=2$) is competitive with the exact orthogonal variant and outperforms vanilla training (SGD + Momentum). In Fig. \ref{plain110}, we present the best test accuracy curve found using a vanilla optimizer ($\mathrm{SGD}$+$\mathrm{Momentum}$), comparable to the best one from \cite{XieXP17} ($66 \%$ test accuracy). As in \cite{XieXP17}, we found training PlainNet-110 with vanilla $\mathrm{Adam}$ challenging. We observe that orthogonal optimization improves training. Even though stochastic optimizers for $s=2$ are too noisy for this task, taking $s>2$ (we used $s=\lceil \frac{d}{\log(d)} \rceil$) provides optimizers competitive with exact orthogonal. Next we present computational advantages of our algorithms over other methods. 

\vspace{-4mm}
\begin{figure}[h]
    \begin{minipage}{1.0\textwidth}
    \includegraphics[width=.45\linewidth]{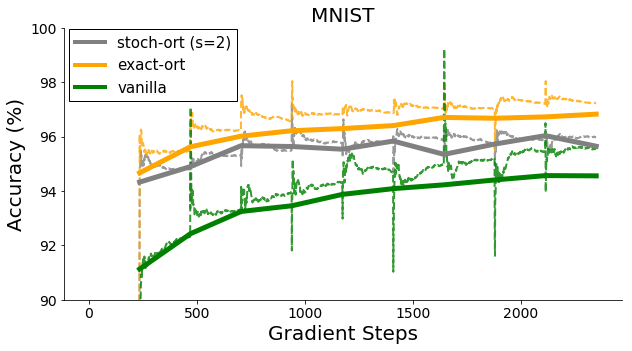}
    \end{minipage}
    \vspace{-4.5mm}
    \caption{Performance of different methods on MNIST. Thin/Bold curves denote Training/Test accuracy respectively.}
\label{mnist}
\end{figure}
\vspace{-4mm}
\begin{figure}[h] 
    \begin{minipage}{1.0\textwidth}
    \includegraphics[width=.45\linewidth]{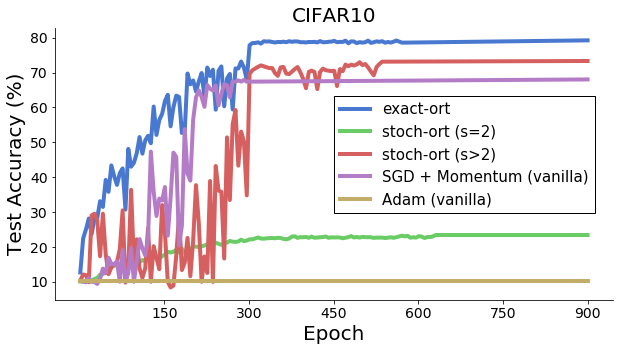}
    \end{minipage}
    \vspace{-4.5mm}    
    \caption{Performance of PlainNet-110 with stochastic optimizer providing improvements over vanilla $\mathrm{SGD}$+$\mathrm{Momentum}$
    }
\label{plain110}
\vspace{-5mm}
\end{figure}

\subsection{Time Complexity}
\label{exp:tc}

\vspace{-4mm}
\begin{center}
\scalebox{0.8}{
 \begin{tabular}{||c | c| c | c | c||} 
 \hline
 \null & \small{$\mathrm{ES}$-$200$} & \small{$\mathrm{ES}$-$400$} & \small{MNIST} & \small{CIFAR} \\ [0.5ex] 
 \hline\hline
 $\mathrm{exact}$-$\mathrm{ort}$ & $>1600$ & $>12800$ & $>200$ & $>49\mathrm{K}$ \\ 
 \hline
  $\mathrm{stoch}$-$\mathrm{ort}$ $(s=2)$  & $\mathbf{<68}$ & $\mathbf{<272}$ & $\mathbf{<8.5}$ & $<2\mathrm{K}$\\ 
 \hline  
 $\mathrm{ortinit}$-$\mathrm{vanRNN}$ $(p=20)$   & $>84$ & $>656$ & N/A & N/A\\   
 \hline  
 $\mathrm{ortinit}$-$\mathrm{vanRNN}$ $(p=10)$   & $>164$ & $>1296$ & N/A & N/A\\ 
 \hline
 $\mathrm{ortinit}$-$\mathrm{vanRNN}$ $(p=8)$   & $>204$ & $>1616$ & N/A & N/A\\  
 \hline
 $\mathrm{ortinit}$-$\mathrm{vanRNN}$ $(p=5)$   & $>324$ & $>2576$ & N/A & N/A\\  
 \hline 
 $\mathrm{ortinit}$-$\mathrm{vanRNN}$ $(p=4)$   & $>404$ & $>3216$ & N/A & N/A\\  
 \hline  
  $\mathrm{stoch}$-$\mathrm{ort}$ $(s=r^{*})$  & N/A & N/A & $<43$ & $\mathbf{<14}$\textbf{K}\\
 \hline 
\end{tabular}}
\end{center}
\vspace{-1mm}
\small Table 1: Comparison of average no of $\mathrm{FLOPS}$ [in $10^{4}$] for orthogonal matrices updates per step for different methods. $\mathrm{ES}$-$h$ stands for the setting from Sec. \ref{sec:rl} with hidden state of size $h$ and $r^{*}=\lceil \frac{d}{\log(d)} \rceil$. Fastest successful runs are bolded.
\normalsize

To abstract from specific implementations, we compare number of $\mathrm{FLOPS}$ per iteration for updates of orthogonal matrices in different methods. See results in Table 1. We see that our optimizers outperform other orthogonal optimization methods, and preserve accuracy, as discussed above.

\vspace{-3mm}
\section{Conclusion}
\vspace{-2mm}
We introduced the first stochastic gradient flows algorithms for ML to optimize on orthogonal manifolds that are characterized by sub-cubic time complexity and maintain representational capacity of standard cubic methods. We provide strong connection with graph theory and show broad spectrum of applications ranging from CNN training for vision to learning RNN-based world models for RL.

\section{Acknowledgements}

Adrian Weller acknowledges support from the David MacKay Newton research fellowship at Darwin College, The Alan Turing Institute under EPSRC grant EP/N510129/1 and U/B/000074, and the Leverhulme Trust via CFI.

\bibliography{ortho}
\bibliographystyle{icml2020}

\newpage
\onecolumn
\section{APPENDIX: Stochastic Flows and Geometric Optimization on the Orthogonal Group}

\subsection{Hyperparameters and Training for CNNs}

\subsubsection{Supervised Learning}
\label{appendix_SL}
In the Plain-110 task on CIFAR10, we performed grid search across the following parameters and values in the orthogonal setting:

\begin{center}
 \begin{tabular}{||c | c||} 
 \hline
 Hyperparameter & Values  \\ [0.5ex] 
 \hline\hline
 learning rate (LR) & \{0.05, 0.1, 0.5\}  \\ 
 \hline
 use bias (whether layers use bias) & \{False, True\}  \\
 \hline
 batch size & \{128, 1024, 8196\} \\
 \hline
 maximum epoch length & \{100, 300, 900\}  \\
 \hline
 scaling on LR for orthogonal integrator & \{0.1, 1.0, 10.0\}  \\ 
 \hline
\end{tabular}
\end{center}

For the vanilla baselines, we used a momentum optimizer with the same settings found in \cite{XieXP17,HeZRS16} (0.9 momentum, 0.1 learning rate, 128 batch size). The learning rate decay schedule occurs when the epoch number is \{3/9, 6/9, 8/9\} of the maximum epoch length.

We also used a similar hyperparameter sweep for the MLP task on MNIST.

\subsubsection{Extra Training Details}
\label{sec:exp_prob}

For the CIFAR10 results from Figure \ref{plain110}, to understand the required computing resources to train PlainNet-110, we further found that \textbf{numerical issues} using the exact integrator could occur when using a naive variant of the matrix exponential. In particular, when the Taylor series truncation  $\sum_{k=0}^{T} \frac{1}{k!} \mathbf{X}^{k}$ for approximating $e^{X}$ is too short (such as even $T = 100$), PlainNet-110 could not reach $\ge 80\%$ \textit{training} accuracy, showing that achieving acceptable precision on the matrix exponential can require a large amount of truncations. An acceptable truncation length was found at T = 200. Furthermore, library functions (e.g. \texttt{tensorflowf.linalg.expm} \cite{higham05}), albeit using optimized code, are still inherently limited to techniques computing these truncations as well.

For the cluster-based stochastic integrators, we set the cluster size for each parameter matrix $\mathbf{M} \in \mathbb{R}^{d, k}$ to be the rounding-up of $\frac{d}{\log d}$. We found that this was an optimal choice, as sizes such as $O(\log d), O(\sqrt{d})$ did not train properly.

\subsection{Orthogonal Optimization for RL - Additional Details}
\label{app:ablation}

We conducted extensive ablation studies to see whether the assumption that one can take upper bound $\tau$ for $\tau^{*}$ (see: Section \ref{sec:algorithm}) of the order 
$O(\frac{s}{d \alpha \beta}) \|h(\Omega)\|_{1}$ for small constants $\alpha^{-1},\beta^{-1}$ is valid. In other words, we want to see whether $\frac{1}{\tau^{*}}$ can be lower-bounded by expressions of the order $\Omega(\frac{d\alpha \beta}{s}\frac{1}{\|h(\Omega)\|_{1}})$.
We took Humanoid environment and the setting as in Section \ref{sec:rl}. Note that by the definition of $\tau^{*}$ we trivially have:
$\frac{1}{\tau^{*}} \geq \rho \frac{1}{\|h(\Omega)\|_{1}}$, 
where $\rho=\frac{\|h(\Omega)\|_{1}}{\gamma_{\Omega}(s,d,h)}$
and $\gamma_{\Omega}(s,d,h)$ is the sum of the $\frac{d}{s} {s \choose 2}$ entries of $h(\Omega)$ with largest absolute values.

In  Fig. \ref{rho_fig} we plot $\rho$ as a function of the number of iterations of the training procedure. Dotted lines correspond to the values $\frac{d}{s}$. The $y$-axis uses $\mathrm{log}$-scale.

We tested different sizes $s=2,4,5,10,20,25,50$ and took the size of the hidden layer to be $200$ (thus $d=200$).
We noticed that for a fixed $s$, values of $\rho$ do not change much over time and can be accurately approximated by constants (in  Fig. \ref{rho_fig} they look almost line the plots of constant functions $y=\mathrm{const}$, even though we observed small perturbations). Furthermore, they can be accurately approximated by renormalized values $\frac{d}{s}$, where renormalization factor $c$ is such that $c^{-1}$ is a small positive constant. That suggests two things: 
\begin{itemize}
    \item $\tau^{*}$ can be in practice upper-bounded by expressions of the form $O(\frac{s}{d \alpha \beta}) \|h(\Omega)\|_{1}$ for small positive constant $\alpha^{-1},\beta^{-1}$ and:
    \item magnitudes of entries of skew-symmetric matrices in applications from Section \ref{sec:rl} tend to be very similar.
\end{itemize}

Of course, as explained in the main body of the paper, those findings enable to further improve speed of our sampling procedures.

\begin{figure}[H]
\label{figure:ablation}
    \begin{minipage}{1.0\textwidth}
    \centering\subfigure[]{\includegraphics[width=.5\linewidth]{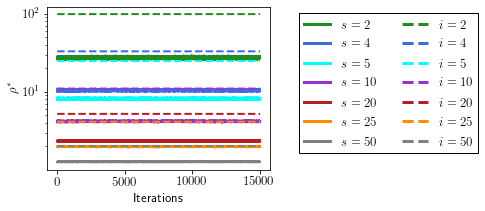}}
    \vspace{-3mm}
    \end{minipage}
\vspace{-4mm}
\caption{\small{Value of $\rho=\frac{\|h(\Omega)\|_{1}}{\gamma_{\Omega}(s,d,h)}$
as a function of the number of iterations of the optimization as in Section \ref{sec:rl} for Humanoid and for different sizes $s=2,4,5,10,20,25,50$. Dotted lines correspond to values $\frac{d}{s}$ that approximate (up to the positive multiplicative constant that is not too small) values of $\rho$. We see that for a fixed $s$, values of $\rho$ almost do not change over the course of optimization and in fact can be accurately approximated by plots of constant functions. We use the $\mathrm{log}$-scale for $y$-axis.}}
\label{rho_fig}
\end{figure}

\subsection{Instability of Deterministic Methods for the Optimization on the Orthogonal Group}
\label{app:instability}

To demonstrate numerical problems of the deterministic optimizers/integrators on the orthogonal group $\mathcal{O}(d)$, we considered the following matrix differential equation on $\mathcal{O}(d)$:
\begin{equation}
\label{sort_de}
\mathbf{\dot{X}}(t)= \mathbf{X}(t)(\mathbf{N}\mathbf{X}(t)^{\top}\mathbf{QX} - \mathbf{X}^{\top}\mathbf{Q}\mathbf{X}(t)\mathbf{N}),  
\end{equation}
where $\mathbf{N} = \mathrm{diag}(n_1,.,,,n_d)$, $\mathbf{Q}=\mathrm{diag}(q_1,...,q_d)$
for some scalars $n_1,...,n_d, q_1,...,q_d \in \mathbb{R}$, and furthermore $n_i \neq n_j$ and $q_i \neq q_j$ for $i \neq j$.
We also assume that $\mathbf{X}(0) \in \mathcal{O}(d)$. Matrix $\Omega(t)=\mathbf{N}\mathbf{X}(t)^{\top}\mathbf{QX} - \mathbf{X}^{\top}\mathbf{Q}\mathbf{X}(t)\mathbf{N}$ is clearly skew-symmetric thus the above differential equation encodes flow evolving on $\mathcal{O}(d)$ (see: Sec. \ref{sec:manifold}).

It can be proven that for all matrices $\mathbf{X} \in \mathcal{O}(d)$, but a set of measure zero the following holds:
\begin{equation}
\mathbf{X}(t) \overset{t}{\rightarrow} \mathbf{P}, 
\end{equation}
where $\mathbf{P}$ is a permutation matrix corresponding to the permutation
$(r_1,...,r_d)$ of $(q_1,...,q_d)$ that maximizes the expression:
\begin{equation}
\label{perm_app_eq}
x_1n_1 + ... + x_dn_d    
\end{equation}
over all permutations $(x_1,...,x_d)$ of $(q_1,...,q_d)$.
Since Expression \ref{perm_app_eq} is maximized for the permutation $(x_1,...,x_d)$
s.t. $x_i < x_j$ iff $n_i < n_j$, we conclude that the flow which is a solution to Eq. \ref{sort_de} can be applied to sort numbers
(e.g. one can take $(n_1,...,n_d)=(1,...,d)$ to sort in the increasing order). 
Furthermore, we can use our techniques to conduct integration.

In our experiments we compared our algorithm (using non-intersecting families with $s=2$) with the deterministic integrator based on exact exponential mapping. We chose $\mathbf{X}(0)$ to be a random orthogonal matrix that we obtained by constructing Gaussian matrix and then conducting Gram-Schmidt orthogonalization and row-renormalization.

\vspace{-3mm}
\begin{center}
\scalebox{0.8}{
 \begin{tabular}{||c | c| c | c | c | c | c | c | c | c | c||} 
 \hline
 \null & $\eta=0.00001$ & $\eta=0.00005$ & $\eta=0.0001$ & $\eta=0.00015$ & $\eta=0.001$ & $\eta=0.0015$ & $\eta=0.01$ & $\eta=0.015$ & $\eta=0.1$ & $\eta=0.15$  \\ [0.5ex] 
 \hline\hline
 $\epsilon$: $\mathrm{stoch}$ & e-13 & 2.0e-12 & 1.5e-12 & 1.8e-12 & 1.65e-12 & 1.2e-12 & 1.78e-12 & 1.3e-12 & 1.45e-12 & 1.3e-12 \\ 
 \hline
 $\mathrm{inv}$: $\mathrm{stoch}$ & 1.0 & 1.0 & 1.0 & 1.0 & 1.0 & 0.8 & 0.67 & 0.6 & 0.59 & 0.58 \\ 
 \hline
$\epsilon$: $\mathrm{exact}$ & e-14 & 2.0e-14 & 1.5e-13 & 1.8e-13 & $\mathrm{nan}$ & $\mathrm{nan}$ & $\mathrm{nan}$ & $\mathrm{nan}$ & $\mathrm{nan}$ & $\mathrm{nan}$\\ 
 \hline  
 $\mathrm{inv}$: $\mathrm{exact}$ & 0.8 & 0.75 & 0.72 & 0.52 & 0.0 & 0.0 & 0.0 & 0.0 & 0.0 & 0.0\\ 
 \hline  
\end{tabular}}
\end{center}
Table 2: Comparison of the stochastic integrator with the exact one on the problem of sorting numbers with flows evolving on $\mathcal{O}(d)$. First two rows correspond to the stochastic integrator and last two to the exact one.
The error $\epsilon$ is defined as: $\epsilon=\|\mathbf{X}_{\mathrm{final}}\mathbf{X}^{\top}_{\mathrm{final}} - \mathbf{I}_d\|_{\mathcal{F}}$.
Value of $\mathrm{inv}$ is the fraction of inverse pairs. For large enough step size exact integrator starts to produce numerical errors that accumulate over time and break integration.

We focused on quantifying numerical instabilities of both methods by conducting ablation studies over different values of step size $\eta>0$. 
For each method we run $n=10$ experiments, each with different random sequence $(q_1,...,q_d)$. We chose: $N=\mathrm{diag}(d,d-1,...,1)$ thus the goal was to sort in the decreasing order.
To conduct sorting, we run $\frac{50.0}{\eta}$ iterations of both algorithms.
Denote by $\mathbf{X}_{\mathrm{final}}$ the matrix obtained by conducting integration. 
We computed $\|\mathbf{X}_{\mathrm{final}}\mathbf{X}^{\top}_{\mathrm{final}} - \mathbf{I}_d\|_{\mathcal{F}}$ to measure the deviation from the orthogonal group $\mathrm{O}(d)$. Matrix $\mathbf{X}_{\mathrm{final}}$ was projected back to the permutation group that was then used to obtain permuted version $(p_1,...,p_d)$ of the original sequence $(q_1,...,q_d)$
The quality of the final result was measured in the number of inverses, i.e. pairs $(p_i, p_j)$ such that $i<j$ but $p_i > p_j$. 
For perfect sorting all the pairs $(p_i, p_j)$ such that $i<j$ are inverses.
As we see in Table 2, if step size is too large exact method produces matrices with infinite field values and the algorithm fails.

\subsection{The Geometry of the Orthogonal Group \& Riemannian Optimization}
\label{app:riemann}

In this section we provide additional technical terminology that we use in the main body of the paper.

\subsubsection{Smooth Curves on Manifolds}
\label{app:smooth}

\begin{definition}[smooth curves on $\mathcal{M}$]
A function $\gamma:I \rightarrow \mathcal{M}$, where $I \subseteq \mathbb{R}$ is an open interval is a smooth curve on $\mathcal{M}$ passing through $\mathbf{p} \in \mathcal{M}$ if there exists $\phi:\Omega_{\mathbf{p}} \rightarrow U_{\mathbf{p}}$ for open subsets $\Omega_{p} \subseteq \mathbb{R}^{d}$, $U_{\mathbf{p}} \subseteq \mathcal{M}$ and $\epsilon>0$ such that the function $\phi^{-1} \circ \gamma:(t-\epsilon,t
+\epsilon) \rightarrow \mathbb{R}^{d}$ is smooth.
\end{definition}

Vectors tangent to smooth curves $\gamma$ on $\mathcal{M}$ passing through fixed point $\mathbf{p} \in \mathcal{M}$ give rise to the linear subspace tangent to $\mathcal{M}$ at $\mathbf{p}$, the \textit{tangent space} $\mathcal{T}_{\mathbf{p}}(\mathcal{M})$ that we define in the main body. 

\subsubsection{Inner Products}
\label{app:inner}

Standard inner products used for $\mathcal{ST}(d,k)$ are: the \textit{Euclidean inner product} defined as: 
$\langle \mathbf{Z}_{1},\mathbf{Z}_{2} \rangle_{\mathrm{e}} = \mathrm{tr}(\mathbf{Z}_{1}^{\top}\mathbf{Z}_{2})$
and the \textit{canonical inner product} given as:
$\langle \mathbf{Z}_{1},\mathbf{Z}_{2} \rangle_{c}=\mathrm{tr}(Z_{1}^{\top}(I-\frac{1}{2}\mathbf{X}\mathbf{X}^{\top}))Z_{2}$ for a tangent space in $\mathbf{X} \in \mathcal{ST}(d,k)$. 

\subsubsection{Representation Theorems for On-Manifold Optimization}
\label{app:representations}

We need the following standard representation theorem:

\begin{theorem}[representation theorem]
\label{repr}
If $\langle \cdot \rangle$ is an inner product defined on the vector space $\mathcal{R}$, then for any linear functional $L:\mathcal{R} \rightarrow \mathbb{R}$ there exists $\mathbf{R} \in \mathcal{R}$ s.t. $\langle \mathbf{R},\mathbf{Q} \rangle=L(\mathbf{Q})$ for any $\mathbf{Q} \in \mathcal{R}$.
\end{theorem}

To apply the above result for on-manifold optimization, we identify: 
\begin{itemize}
    \item $L$ with the directional derivative operator related to the function $F$ being optimized,
    \item $\mathcal{R}$ with the tangent space,
    \item $\mathbf{R}$ with the Riemannian gradient.
\end{itemize}

\subsection{Theoretical Results for Sampling Algorithms}
\label{sec:theory_sampling}

Below we prove all the theoretical results from Section \ref{sec:algorithm}.

\begin{lemma}
\label{lemma:uniform}
We state two useful combinatorial facts and one of their consequences: 
\begin{itemize}
    \item $|\mathcal{T}_s| = \dfrac{d!}{(s!)^{d/s}(\frac{d}{s})!} $
    \item Each edge appears in $W = \dfrac{(d-2)!}{(\frac{d-s}{s})!(s!)^{\frac{d-s}{s}}(s-2)!}$ tournaments of $|\mathcal{T}_s|$
    \item Therefore, for $p \sim \mathcal{U}(\mathcal{T}_s)$, $\dfrac{1}{p_T}\mathbf{M}_{\mathcal{T}_s} = \dfrac{d-1}{s-1}\mathbf{J}_d$
\end{itemize}
\end{lemma}
\begin{proof}

\begin{itemize}
    \item  To compute $|\mathcal{T}_s|$, we can use the way we sample them: we choose a random permutation and take the $s$ first vertices to be the first connected component, the $s$ next vertices to be the second etc... This way, multiple random permutations will lead to the same tournament. More precisely, exactly $\dfrac{d}{s}!(s!)^{\frac{d}{s}}$ permutations lead to the same tournament. 
 
 Therefore $|\mathcal{T}_s| = \dfrac{d!}{(s!)^{d/s}(\frac{d}{s})!}$.
 \item  By symmetry, we know that each edge appears in the same number of tournaments of $\mathcal{T}_s$. Let $W$ be this number. Let $N_T$ be the number of edges in the tournament $T$. We have that $N_T = \dfrac{d}{s} {s \choose 2}$.
 Therefore $\sum_{T \in \mathcal{T}_s} N_T = |\mathcal{T}_s|\dfrac{d}{s} {s \choose 2}$. We also have that $\sum_{T \in \mathcal{T}_s} N_T = W{d \choose 2}$.
 Therefore $|\mathcal{T}_s|\dfrac{d}{s} {s \choose 2} = W{d \choose 2}$
 That gives:
 \begin{align*}
    W &= \dfrac{d!\dfrac{d}{s}{s \choose 2}}{{d \choose 2}(\dfrac{d}{s})!(s!)^{d/s}}\\
    &= \dfrac{(d-2)!s(s-1)}{(\dfrac{d}{s}-1)!(s!)^{d/s}} \\
    &= \dfrac{(d-2)!}{(\dfrac{d-s}{s})!(s!)^{\dfrac{d-s}{s}}(s-2)!}\\
    W &= \frac{(d-2)!}{(\frac{d-s}{s})!(s!)^{\frac{d-s}{s}}(s-2)!}
\end{align*}
\item For $p \sim \mathcal{U}(\mathcal{T}_s)$, $p_T = \dfrac{1}{|\mathcal{T}_s|}$. Therefore, $\dfrac{1}{p_T}\mathbf{M}_{\mathcal{T}_s} = \dfrac{|\mathcal{T}_s|}{W}\mathbf{J}_d = \dfrac{d-1}{s-1}\mathbf{J}_d$ 
\end{itemize}
\end{proof}

\subsubsection{Proof of Lemma \ref{importance_sampling}}

Below we prove Lemma \ref{importance_sampling} from the main body of the paper.

\begin{proof}
Let $\Omega \in \mathrm{Sk}(d)$ be a skew-symmetric matrix . Fix a family $\mathcal{T}$ of subtournaments of $T(\Omega)$. We aim to show that the distribution $\mathcal{P}$ over $\mathcal{T}$ minimizing the variance $\mathrm{Var}(\widehat{\Omega})=\mathbb{E}
[\|\widehat{\Omega} - \Omega \|^{2}_{\mathcal{F}}]$ among unbiased distributions of the form given by equations \ref{approx} and \ref{eqn:estimator}, satisfies: $p_T \sim \sqrt{\sum_{e \in E(G_T)} w_{e}^{2}}$, where $w_e$ is the weight of edge $e$. 

The constraint on the scalars $\{p_T\}_{T \in \mathcal{T}}$ is simply that the family $\{p_T\}_{T \in \mathcal{T}}$ forms a valid probability distribution. The unbiasedness is guaranteed by the equations \ref{approx} and \ref{eqn:estimator}.

The variance rewrites:
\newcommand{\om}{\widehat{\Omega}}
$$  \mathrm{Var}(\widehat{\Omega}) = \mathbb{E}
[\|\widehat{\Omega} - \Omega \|^{2}_{\mathcal{F}}]  =  \mathbb{E}[\|\widehat{\Omega}\|^{2}_{\mathcal{F}}] - \|\Omega\|^{2}_{\mathcal{F}}\\
$$
Then we consider the following functional,
\begin{align*}
   f(\mathcal{P}) &= \mathbb{E}_{\mathcal{P}}[\|\widehat{\Omega}\|^{2}_{\mathcal{F}}]\\
   &= \sum_{T \in \mathcal{T}} p_T \cdot \left\| \frac{1}{p_{T}}M_{\mathcal{T}} \odot \Omega[T] \right\|_F^2 \\
   &= \sum_{T \in \mathcal{T}} \frac{2}{p_{T}} \sum_{(i,j) \in E(G_T)} M_{\mathcal{T}}[i,j]^2 \cdot \Omega[i,j]^2 & \text{the order $i,j$ does not matter}\\
\end{align*}

We minimize the functional $f$ on the convex open domain $\{ \mathcal{P}=\{p_T\}_{T \in \mathcal{T}} \in  (\mathbb{R}_{>0})^\mathcal{T}, \sum_{T \in \mathcal{T}} p_T = 1 \}$ on which $f$ is convex. The Lagrangian has the form:
\begin{align*}
    L( \mathcal{P} , \lambda) = &\sum_{T \in \mathcal{T}} \frac{2}{p_{T}} \sum_{(i,j) \in E(G_T)} M_{\mathcal{T}}[i,j]^2 \cdot \Omega[i,j]^2 + 2 \lambda \left(\sum_{T \in \mathcal{T}} p_T - 1\right)
\end{align*}
and the global optimum can be found from equations:

\begin{align*}
    \frac{\partial}{\partial p_T} L(\mathcal{P}, \lambda) = -\frac{2}{p^2_T} \sum_{(i,j) \in E(G_T)} M_{\mathcal{T}}[i,j]^2 \cdot \Omega[i,j]^2 + 2\lambda = 0
\end{align*}

We finally obtain the optimal $\mathcal{P}$:
\begin{equation}
    p_T = \frac{\sqrt{\sum_{(i,j) \in E(G_T)} \left(M_{\mathcal{T}} \odot  \Omega\right)[i,j]^2 }}{Z} 
\end{equation} where $Z =  \sum_{T \in \mathcal{T}} p_T$.

We find that the smallest variance is then given by:

$$\text{Var}^*\left(\widehat{\Omega}\right) = 2 \cdot \left(\sum_{T \in \mathcal{T}} \sqrt{\sum_{(i,j) \in E(G_T)} \left(M_{\mathcal{T}} \odot  \Omega\right)[i,j]^2 }\right)^2 - \|\Omega\|^{2}_{\mathcal{F}}$$

In case of homogeneous families, $M_\mathcal{T}$ has identical coefficients and the constant $M_\mathcal{T}$ vanishes into the normalization constant $Z$. 
\end{proof}

\subsubsection{Proof of Lemma \ref{expected_number_of_trials}}

Below we prove Lemma \ref{expected_number_of_trials} from the main body of the paper.
\begin{proof}
Let $A_k$ be the random variable which is $1$ if the $k$th sample is accepted and $0$ otherwise and $T_k$ be the $k$th sampled tournament. $A_k$ are iid Bernoulli variables of parameter $\dfrac{\lambda}{|\mathcal{T}_s|}$.

\begin{align*}
    \pr{A_1 = 1} &= \sum_{T \in \mathcal{T}_s}\pr{A_1 = 1 | T_1 = T}\pr{T_1 = T} \\
    &= \dfrac{1}{|\mathcal{T}_s|} \sum_{T \in \mathcal{T}_s}\pr{A_1 = 1 | T_1 = T} \\
    &= \dfrac{1}{|\mathcal{T}_s|} \sum_{T \in \mathcal{T}_s} q^h_T = \dfrac{\lambda }{|\mathcal{T}_s|} \sum_{T \in \mathcal{T}_s} p^h_T \\
    &= \dfrac{\lambda }{|\mathcal{T}_s|}
\end{align*}

The number of trials before a sample is accepted is $\min \{k | A_k = 1\}$. This random variable follows a Poisson distribution of parameter $\dfrac{|\mathcal{T}_s|}{\lambda}$. Therefore, the expected number of trials before a sample is accepted is $\dfrac{|\mathcal{T}_s|}{\lambda}$

\end{proof}

\subsubsection{Proof of Lemma \ref{probability_formula}}

Below we prove Lemma \ref{probability_formula} from the main body of the paper.
\begin{proof}
By definition, $p^{h}_{T} \sim h(G_T)$. Let call $\alpha$ the proportionality factor. 
Then 
\begin{align*}
    \alpha^{-1} &= \sum_{T \in \mathcal{T}_s}h(G_T) = \sum_{T \in \mathcal{T}_s}\sum_{e \in E(G_T)}h(w_e)\\
    &= \sum_{T \in \mathcal{T}_s}\sum_{i < j}h(w_{(i,j)})\ind{(i,j) \in E(G_T)} \\
    &= \sum_{i <j}h(w_{(i,j)})\sum_{T \in \mathcal{T}_s}\ind{(i,j) \in E(G_T)} \\
     &= W\sum_{i < j}h(w_{(i,j)}) = W\dfrac{1}{2}\sum_{i,j}h(w_{(i,j)}) \\
    &= W\dfrac{1}{2}\sum_{i,j}h(\Omega_{(i,j)}) = \dfrac{W\|h(\Omega)\|_1}{2}
\end{align*}

As $W = \sum_{i < j}h(w_{(i,j)})\ind{(i,j) \in E(G_T)}$. The computation of $W$ is done in the proof of Lemma \ref{lemma:uniform}.
\end{proof}

\subsubsection{Proof of Theorem \ref{iter-time-complexity}}

Below we prove Theorem \ref{iter-time-complexity} from the main body of the paper.
\begin{proof}
The time complexity results is a direct consequence of Lemma \ref{expected_number_of_trials}. We just need to prove that Algorithm \ref{Alg:asebo} returns a sample of $\mathcal{P}^{h}(\mathcal{T}_s)$. We use the random variables $A_k$ and $T_k$ defined in the proof of Lemma \ref{expected_number_of_trials}. Let $A$ be the output of Algorithm \ref{Alg:asebo}.

Let $T \in \mathcal{T}_s$. We have to check that $\pr{A = T} = p^h_T$. For this, we notice that $\{A = T\} = \cup_{k=1}^{+\infty}\{A_k = 1 \cap T_k = T \cap_{i=1}^{k-1} A_i = 0\}$. These events being disjoints, we have:
\begin{align*}
    \pr{A = T} &= \sum_{k=1}^{+\infty} \pr{A_k = 1 \cap T_k = T \cap_{i=1}^{k-1} A_i = 0}\\
     &= \sum_{k=1}^{+\infty} \pr{A_k = 1 \cap T_k = T | \cap_{i=1}^{k-1} A_i = 0}\pr{\cap_{i=1}^{k-1} A_i = 0}\\
     &= \sum_{k=1}^{+\infty} \pr{A_k = 1 \cap T_k = T | \cap_{i=1}^{k-1} A_i = 0}\left(1- \dfrac{\lambda}{|\mathcal{T}_s|}\right)^{k-1}\\
     &= \sum_{k=1}^{+\infty} \pr{A_k = 1 | T_k = T \cap_{i=1}^{k-1} A_i = 0}\dfrac{1}{|\mathcal{T}_s|}\left(1- \dfrac{\lambda}{|\mathcal{T}_s|}\right)^{k-1}\\
     &= \sum_{k=1}^{+\infty} q^h_T\dfrac{1}{|\mathcal{T}_s|}\left(1- \dfrac{\lambda}{|\mathcal{T}_s|}\right)^{k-1}\\
     &= p^h_T\dfrac{\lambda}{|\mathcal{T}_s|}\sum_{k=0}^{+\infty} \left(1- \dfrac{\lambda}{|\mathcal{T}_s|}\right)^{k}\\
     \pr{A = T} &= p^h_T
\end{align*}
Therefore Algorithm \ref{Alg:asebo} samples from $\mathcal{P}^h(\mathcal{T}_s)$.
\end{proof}

\subsubsection{Estimating $\|h(\Omega)\|_1$}
\label{app:mc}
Denote $n={d \choose 2}$.
Consider matrix $h(\Omega) \in \mathbb{R}^{d \times d}$. We will approximate 
$\|h(\Omega)\|_{1}$ as:
\begin{equation}
X = \sum_{i,j} X_{i,j},    
\end{equation}
for $1 \leq i < j \leq d$
and where $X_{i,j}=\frac{n}{r}h(\Omega_{i,j})$ with probability $\frac{r}{n}$ and $X_{i,j}=0$ otherwise. Note that $\mathbb{E}[X]=\|h(\Omega)\|_{1}$ and furthermore the expected number $R$ of nonzero entries $X_{i,j}$ is clearly $r$. Now it suffices to notice that $R$ is strongly concentrated around its mean using standard concentration inequalities (such as Azuma's inequality). Furthermore, for any $a>0$, by Azuma's inequality, we have:
\begin{equation}
\mathbb{P}[X-\mathbb{E}[X] > a] \leq \mathrm{exp}(-\frac{a^{2}}{2(\frac{n}{r})^{2}\sum_{i,j}h^{2}(\Omega_{i,j})}).
\end{equation}

The upper bound is clearly smaller than $\mathrm{exp}(-(\frac{\epsilon \alpha \beta r}{3d})^{2})$ for $a=\epsilon \|h(\Omega)\|_{1}$ and $(\alpha \beta, h)$-balanced $\Omega$. That directly leads to the results regarding approximating $\|h(\Omega)\|_{1}$ by sub-sampling $\Omega$ from the main body of the paper.

\subsection{Variance Results}
\label{app:var}

Below we present variance results of the estimators of skew-symmetric matrices $\Omega$ studied in the main body of the paper.
\begin{lemma}[Variance of $h$-regular estimators]
\label{general_variance}
The variance of an estimator $\hat{\Omega}$ following an $h$-regular distribution over $\mathcal{T}_s$ is $$\mathrm{Var}(\hat{\Omega}) = \dfrac{\|h(\Omega)\|_1}{2W}\sum_{T \in \mathcal{T}_s}\dfrac{\|\Omega[T]\|_\mathcal{F}^2}{h(G_T)}- \|\Omega\|^2_\mathcal{F}$$

\end{lemma}
\begin{proof}
    \begin{align*}
        \mathrm{Var}(\hat{\Omega}) &= \sum_{T \in \mathcal{T}_s} p^h_T\|\Omega_T\|^2_\mathcal{F} - \|\Omega\|^2_\mathcal{F}\\
        &= \sum_{T \in \mathcal{T}_s} \dfrac{1}{W^2p^h_T}\|\Omega[T]\|_\mathcal{F}^2- \|\Omega\|^2_\mathcal{F}\\
        &= \sum_{T \in \mathcal{T}_s}\dfrac{\|h(\Omega)\|_1}{2Wh(G_T)}\|\Omega[T]\|_\mathcal{F}^2- \|\Omega\|^2_\mathcal{F}\\
        &= \dfrac{\|h(\Omega)\|_1}{2W}\sum_{T \in \mathcal{T}_s}\dfrac{\|\Omega[T]\|_\mathcal{F}^2}{h(G_T)}- \|\Omega\|^2_\mathcal{F}\\
    \end{align*}
\end{proof}
\begin{lemma}
\label{squared_variance}
Let $\hat{\Omega}$ be the $h$-regular estimator over $\mathcal{T}_s$ where $h$ is the squared function. Then $\mathrm{Var}(\hat{\Omega}) = \dfrac{d-s}{s-1}\|\Omega\|^2_\mathcal{F}$
\end{lemma}
\begin{proof}
Using lemma \ref{general_variance} with $h$ being the squared function gives:
\begin{align*}
    &= \dfrac{\|\Omega\|_\mathcal{F}^2}{W}\sum_{T \in \mathcal{T}_s}1 - \|\Omega\|_\mathcal{F}^2 \quad \text{as $2h(G_T) = \|\Omega[T]\|^2_\mathcal{F}$}\\
    &= \|\Omega\|_\mathcal{F}^2 \left(\dfrac{|\mathcal{T}_s|}{W} - 1\right)\\
     &= \|\Omega\|_\mathcal{F}^2 \left( \dfrac{d-1}{s-1} - 1\right)\quad \text{as seen in the proof of Lemma \ref{lemma:uniform}}
    \intertext{Therefore:}
    \mathrm{Var}(\hat{\Omega}) &= \dfrac{d-s}{s-1}\|\Omega\|_\mathcal{F}^2
\end{align*}
\end{proof}

\begin{lemma}
\label{uniform_variance}
Let $\hat{\Omega}$ be uniformly distributed over $\mathcal{T}_s$. Then $\mathrm{Var}(\hat{\Omega}) = \dfrac{d-s}{s-1}\|\Omega\|^2_\mathcal{F}$
\end{lemma}\begin{proof}
Let $\hat{\Omega}$ be uniformly distributed over $\mathcal{T}_s$.
We have:
\begin{align*}
    \mathrm{Var}(\hat{\Omega}) &= \sum_{T \in \mathcal{T}_s}\dfrac{1}{|\mathcal{T}_s|}\|\Omega_T\|_\mathcal{F}^2 - \|\Omega\|_\mathcal{F}^2 \\
    &= \dfrac{1}{|\mathcal{T}_s|}\sum_{T \in \mathcal{T}_s}\dfrac{(d-1)^2}{(s-1)^2}\|\Omega[T]\|_\mathcal{F}^2 - \|\Omega\|_\mathcal{F}^2 \\
    &= \dfrac{1}{|\mathcal{T}_s|}\dfrac{(d-1)^2}{(s-1)^2}\sum_{T \in \mathcal{T}_s}\sum_{(i,j) \in T}2\Omega_{i,j}^2 - \|\Omega\|_\mathcal{F}^2 \\
    &= \dfrac{1}{|\mathcal{T}_s|}\dfrac{(d-1)^2}{(s-1)^2}\sum_{T \in \mathcal{T}_s}\sum_{i < j}2\Omega_{i,j}^2\ind{(i,j) \in E(G_T)} - \|\Omega\|_\mathcal{F}^2 \\
    &= \dfrac{1}{|\mathcal{T}_s|}\dfrac{(d-1)^2}{(s-1)^2}\sum_{i < j}2\Omega_{i,j}^2\sum_{T \in \mathcal{T}_s}\ind{(i,j) \in E(G_T)} - \|\Omega\|_\mathcal{F}^2 \\
    &= \dfrac{W}{|\mathcal{T}_s|}\dfrac{(d-1)^2}{(s-1)^2}\sum_{i,j}\Omega_{i,j}^2 - \|\Omega\|_\mathcal{F}^2 \\
    &=\dfrac{d-s}{s-1} \|\Omega\|_\mathcal{F}^2 \quad \text {as $\dfrac{W}{\mathcal{T}_s} = \dfrac{s-1}{d-1}$ as seen in the proof of Lemma \ref{lemma:uniform}}\\
\end{align*}
So $\mathrm{Var}(\hat{\Omega}) = \dfrac{d-s}{s-1} \|\Omega\|_\mathcal{F}^2$
\end{proof}

\subsection{The Combinatorics of Domain-Optimization for Sampling Subtournaments}
\label{app:extra_combinatorica}

In this section we provide additional theoretical results regarding variance of certain classes of the proposed estimators of skew-symmetric matrices $\Omega$ and establish deep connection with challenging problems in graph theory and combinatorics. We will be interested in particular in shaping the family of tournaments $\mathcal{T}$ on-the-fly to obtain low-variance estimators.
Even though we did not need these extensions to obtain the results presented in the main body of the paper, we discuss them in more detail here due to the interesting connections with combinatorial optimization. We will focus here on non-intersecting families $\mathcal{T}$ and $s=2$. Thus the corresponding undirected graphs are just matchings and they altogether cover all the edges of the base complete undirected weighted graph $G_{T(\Omega)}$.

\subsubsection{More on the variance}

We will denote the family of all these matchings as $\mathcal{M}$.
and start with function $h:\mathbb{R} \rightarrow \mathbb{R}$ given as: $h(x)=|x|$. The following is true:

\begin{lemma}[variance of matching-based estimators for non-intersecting families and $h(x)=|x|$]
Given a skew-symmetric matrix $\Omega$ and the corresponding complete weighted graph $G_{T(\Omega)}$ with the set of edge-weights $\{w_{e}\}_{e \in E(G_{T(\Omega)})}$, the variance/mean squared error of the unbiased estimator $\widehat{\Omega}$ applying function $h(x)=|x|$ and family of matchings $\mathcal{M}$ satisfies:
\begin{align}
\label{variance}
\begin{split}
\mathrm{MSE}(\widehat{\Omega})=\mathrm{Var}(\widehat{\Omega}) \\ =
\mathbb{E}[\|\widehat{\Omega}-\Omega\|_{\mathcal{F}}^{2}]=
K\sum_{e \in E(G_{T(\Omega)})}\frac{w_{e}^{2}}{K(e)} - \|\Omega\|_{\mathcal{F}}^{2},
\end{split}
\end{align}
where $K(e)$ stands for the sum of absolute values of weights of the edges of the matching $m \in \mathcal{M}$ containing e and $K$ for the sum of all the absolute values of all the weights. 
\end{lemma}

\begin{proof}
We have the following for $\mathbf{V}_{m}$ defined as:
$\mathbf{V}_{m} = \sum_{e \in m} \frac{|a_{i,j}|}{K_{m}}K\mathrm{sgn}(a_{i,j})\mathbf{H}_{i,j}$,where $m$ stands for the matching, and $K_m$ is the sum of weights of matching $m$:
\begin{align}
\begin{split}
\mathbb{E}[\|\widehat{\Omega}-\Omega\|_{\mathcal{F}}^{2}] = 
\mathbb{E}[\|\widehat{\Omega}\|_{\mathcal{F}}^{2}] - \|\Omega\|_{\mathcal{F}}^{2} \\ =
\sum_{m \in \mathcal{M}} p_{m} \|\mathbf{V}_{m}\|_{F}^{2} - \|\Omega\|_{\mathcal{F}}^{2}
= \sum_{m \in \mathcal{M}}p_{m}\sum_{e \in m}
\frac{w_{e}^{2}}{K_{m}^{2}}K^{2} -  \|\Omega\|_{\mathcal{F}}^{2}= \\
K^{2}\sum_{m \in \mathcal{M}} \frac{K_{m}}{K}
\sum_{e \in m}\frac{w_{e}^{2}}{K_{m}^{2}} - \|\Omega\|_{\mathcal{F}}^{2} = \\ 
K\sum_{m \in \mathcal{M}}\frac{1}{K_{m}}\sum_{e \in m} w_{e}^{2} -  \|\Omega\|_{\mathcal{F}}^{2} 
= K\sum_{e \in E(G_{T(\Omega)})}\frac{w_{e}^{2}}{K(e)} - \|\Omega\|_{\mathcal{F}}^{2},
\end{split}
\end{align}
where $p_{m}$ is the probability of choosing matching $m \in \mathcal{M}$, i.e.
$p(m) = \frac{\sum_{e \in m} |w_{e}|}{K}$=$\frac{K_m}{K}$.
\end{proof}

Thus the variance minimization problem reduces to finding a family of matchings $\mathcal{M}$ which minimizes $\sum_{e \in E(G_{T(\Omega)})}\frac{w_{e}^{2}}{K(e)}$. 

Let us list a couple of observations. First, if every matching is a single edge (that would correspond to conducting \textbf{exactly one} multiplication by Givens rotation per iteration of the optimization procedure using an estimator) the variance is the largest. Intuitively speaking, we would like to have in $\mathcal{M}$ lots of heavy-weight matchings. ideally if $\mathcal{M}$ consists of just one matching covering all nonzero-weight edges (the zero-weight edges can be neglected) the variance is the smallest and in fact equals to $0$ since then we take entire matrix $\Omega$. There are lots of heuristics that can be used such as taking maximum weight matching (see: \cite{vazirani}) in $G_{T(\Omega)}$ as the first matching, delete it from graph, take the second largest maximum weight matching and continue to construct entire $\mathcal{M}$. Since finding maximum weight matching requires nontrivial computational time such an approach would work best if we reconstruct $\mathcal{M}$ periodically, as opposed to doing it in every single step of the optimization procedure. Interestingly, it can be shown that this algorithm, even though working very well in practice accuracy-wise, does not minimize the variance (one can find counterexamples with graphs as small as of six vertices). The following is true:

\begin{lemma}[Variance minimization vs. NP-hardness] Given a weighted and undirected graph $G$, the problem of finding a partition of the edges into matchings $\mathcal{M}$ which minimizes $\sum_{e \in E(G)}\frac{w_{e}^{2}}{K(e)}$ is NP-hard.
\end{lemma}

\begin{proof}
There is a one-to-one correspondence between partitions of the edges into matchings $\mathcal{M}$ and edge-colorings. Thus, we will reduce to the problem of computing the chromatic index of an arbitrary graph $G$, which is known to be NP-complete (see \cite{holyer}).\\ 

Take an arbitrary  $G$ and set all its weights $w_e$ equal to $1$. Then we claim the optimal objective value of the optimization problem is the chromatic index of $G$. Indeed,
\begin{align*}
\sum_{e\in E(G)} \frac{w_e^2}{K(e)} = \sum_{e\in E(G)} \frac{1}{K(e)} =\\ \sum_{e\in E(G)} \frac{1}{\#\{e' \in m: e \in m\}} = \#\mathcal{M}
\end{align*}
(where $\#A$ denotes the cardinality of $A$). Thus the expression which minimizes the sum on the LHS is the smallest possible cardinality of the set $\mathcal{M}$, which is the chromatic index of $G$, and thus we have completed the reduction.
\end{proof}

The above result shows an intriguing connection between stochastic optimization on the orthogonal group and graph theory. Notice that we know (see: Lemma \ref{importance_sampling}) that under assumptions regarding estimator from Lemma \ref{importance_sampling}, the optimal variance is achieved if $p_{m}$ is proportional to the square root of the sum of squares of the weights of all its edges. Thus one can instead use such a distribution $\{p_{m}\}_{m \in \mathcal{M}}$ instead the one generated by function $h$. It is an interesting question whether optimizing family of matchings $\mathcal{M}$ (thus we still focus on the case $s=2$) in such a setting can be done in the polynomial time. We leave it to future work.

\subsubsection{Distributed computations for on-manifold optimization}

The connection with maximum graph matching problem suggests that one can apply distributed computations to construct on-the-fly families $\mathcal{M}$ used to conduct sampling. Maximum weight matching is one of the most-studied algorithmic problems in graph theory and the literature on fast distributed optimization algorithms constructing approximations of the maximum weight matching is voluminous (see for instance: \cite{czumaj},\cite{lattanzi},\cite{assadi}).
Such an approach might be particularly convenient if we want to update $\mathcal{M}$ at every single iteration of the optimization procedure and dimensionality $d$ is very large.

\subsubsection{On-manifold optimization vs. graph sparsification problem}

Finally, we want to talk about the connection with graph sparsification techniques. Instead of partitioning into matchings the original graph $G_T$, one can instead sparsify $G_T$ first and then conduct partitioning into matchings of the sparsified graph. This strategy can bypass potentially expensive computations of the heavy-weight matchings in the original dense graph by those in its sparser compact representation. That leads to the theory of graph sparsification and graph sketches \cite{sachdeva} that we leave to future work.

\subsection{Theorem \ref{osgd} Proof}

\begin{proof}
Consider the $i$-th step of the update rule. Denote $g (\eta) = F ( \exp (\eta \widehat{\Omega}_i) \mathbf{X}_i )$. Then by a chain rule we get
\begin{equation*}
    g' (\eta) = \langle \nabla F (\exp ( \eta \widehat{\Omega}_i ) \mathbf{X}_i) , \exp ( \eta \widehat{\Omega}_i ) \widehat{\Omega}_i \mathbf{X}_i \rangle_{\mathrm{e}}
\end{equation*}
Next we deduce
\begin{align}
    &| g' (\eta) - g' (0) | = | \langle \nabla F (\exp ( \eta \widehat{\Omega}_i ) \mathbf{X}_i) , \exp ( \eta \widehat{\Omega}_i ) \widehat{\Omega}_i \mathbf{X}_i \rangle_{\mathrm{e}} - \langle \nabla F ( \mathbf{X}_i) , \widehat{\Omega}_i \mathbf{X}_i \rangle_{\mathrm{e}} | \nonumber \\
    &= | \langle \nabla F (\exp ( \eta \widehat{\Omega}_i ) \mathbf{X}_i) , \exp ( \eta \widehat{\Omega}_i ) \widehat{\Omega}_i \mathbf{X}_i \rangle_{\mathrm{e}} - \langle \nabla F ( \mathbf{X}_i) , \exp ( \eta \widehat{\Omega}_i ) \widehat{\Omega}_i \mathbf{X}_i \rangle_{\mathrm{e}} + \langle \nabla F ( \mathbf{X}_i) , \exp ( \eta \widehat{\Omega}_i ) \widehat{\Omega}_i \mathbf{X}_i \rangle_{\mathrm{e}} \nonumber \\
    &- \langle \nabla F ( \mathbf{X}_i) , \widehat{\Omega}_i \mathbf{X}_i \rangle_{\mathrm{e}} | \nonumber \\
    &\leq | \langle \nabla F (\exp ( \eta \widehat{\Omega}_i ) \mathbf{X}_i) - \nabla F (\mathbf{X}_i) , \exp ( \eta \widehat{\Omega}_i ) \widehat{\Omega}_i \mathbf{X}_i \rangle_{\mathrm{e}} | + | \langle \nabla F ( \mathbf{X}_i) , \exp ( \eta \widehat{\Omega}_i ) \widehat{\Omega}_i \mathbf{X}_i - \widehat{\Omega}_i \mathbf{X}_i \rangle_{\mathrm{e}} | \nonumber \\
    &\leq \| \nabla F (\exp ( \eta \widehat{\Omega}_i ) \mathbf{X}_i) - \nabla F (\mathbf{X}_i) \|_\mathcal{F} \| \exp ( \eta \widehat{\Omega}_i ) \widehat{\Omega}_i \mathbf{X}_i \|_\mathcal{F} + \| \nabla F ( \mathbf{X}_i) \|_\mathcal{F} \| \exp ( \eta \widehat{\Omega}_i ) \widehat{\Omega}_i \mathbf{X}_i - \widehat{\Omega}_i \mathbf{X}_i \|_\mathcal{F} \label{c-sch} \\
    &= \| \nabla F (\exp ( \eta \widehat{\Omega}_i ) \mathbf{X}_i) - \nabla F (\mathbf{X}_i) \|_\mathcal{F} \| \widehat{\Omega}_i \|_\mathcal{F} + \| \nabla F ( \mathbf{X}_i) \|_\mathcal{F} \| \exp ( \eta \widehat{\Omega}_i ) \widehat{\Omega}_i - \widehat{\Omega}_i \|_\mathcal{F} \label{orth1} \\
    &\leq L \| \exp ( \eta \widehat{\Omega}_i ) \mathbf{X}_i - \mathbf{X}_i \|_\mathcal{F} \| \widehat{\Omega}_i \|_\mathcal{F} + \| \nabla F ( \mathbf{X}_i) \|_\mathcal{F} \| \exp ( \eta \widehat{\Omega}_i ) \widehat{\Omega}_i - \widehat{\Omega}_i \|_\mathcal{F} \label{lipsch} \\
    &= L \| \exp ( \eta \widehat{\Omega}_i ) - \mathbf{I}_d \|_\mathcal{F} \| \widehat{\Omega}_i \|_\mathcal{F} + \| \nabla F ( \mathbf{X}_i) \|_\mathcal{F} \| \exp ( \eta \widehat{\Omega}_i ) \widehat{\Omega}_i - \widehat{\Omega}_i \|_\mathcal{F} \label{orth2} \\
    &\leq L \| \exp ( \eta \widehat{\Omega}_i ) - \mathbf{I}_d \|_\mathcal{F} \| \widehat{\Omega}_i \|_\mathcal{F} + \| \nabla F ( \mathbf{X}_i) \|_\mathcal{F} \| \exp ( \eta \widehat{\Omega}_i ) - \mathbf{I}_d \|_\mathcal{F} \| \widehat{\Omega}_i \|_\mathcal{F} \label{submul}
\end{align}
where a) in transition \ref{c-sch} we use Cauchy-Schwarz inequality, b) in \ref{orth1}, \ref{orth2} we use invariance of the Frobenius norm under orthogonal mappings, c) in \ref{lipsch} we use \ref{lipschitz} and d) in \ref{submul} we use sub-multiplicativity of Frobenius norm. We further derive that
\begin{align*}
    \| \nabla F (\mathbf{X}_i) \|_\mathcal{F} &\leq \| \nabla F (\mathbf{X}_i) - \nabla F( \mathbf{I}_d ) \|_\mathcal{F} + \| \nabla F (\mathbf{I}_d) \|_\mathcal{F} \leq L \| \mathbf{X}_i - \mathbf{I}_d \|_\mathcal{F} + \| \nabla F (\mathbf{I}_d) \|_\mathcal{F} \\
    &\leq L (\| \mathbf{X}_i \|_\mathcal{F} + \| \mathbf{I}_d \|_\mathcal{F}) + \| \nabla F (\mathbf{I}_d) \|_\mathcal{F} = 2 L \sqrt{d} + \| \nabla F (\mathbf{I}) \|_\mathcal{F}
\end{align*}
where we use that $\| \mathbf{X}_i \|_\mathcal{F} = \| \mathbf{I}_d \|_\mathcal{F} = \sqrt{d}$ due to orthogonality. Now we have
\begin{equation} 
    | g' (\eta) - g' (0) | \leq \biggl( (2 \sqrt{d} + 1) L + \| \nabla F ( \mathbf{I}_d ) \|_\mathcal{F} \biggr) \| \widehat{\Omega}_i \|_\mathcal{F} \cdot \| \exp ( \eta \widehat{\Omega}_i ) - \mathbf{I}_d \|_\mathcal{F} \label{lipsch_g1}
\end{equation}

Next, we employ Theorem 12.9 from \cite{gallier} which states that, due to its skew-symmetry, $\widehat{\Omega}_i$ can be decomposed as $\widehat{\Omega}_i = \mathbf{P} \mathbf{E} \mathbf{P}^\top$ where $\mathbf{P} \in \mathcal{O} (d)$ and $\mathbf{E}$ is a block-diagonal matrix of form:
\begin{equation*}
    \mathbf{E} = \begin{bmatrix} \mathbf{E}_1 & & \\ & \dots & \\ & & \mathbf{E}_p \end{bmatrix}
\end{equation*}
such that each block $\mathbf{E}_j$ is either $\begin{bmatrix} 0 \end{bmatrix}$ or a two-dimensional matrix of form
\begin{equation*}
    \mathbf{E}_j = \begin{bmatrix} 0 & - \mu_j \\ \mu_j & 0 \end{bmatrix}
\end{equation*}
for some $\mu_j \in \mathbb{R}$. From this we deduce that
\begin{equation*}
    \exp ( \eta \widehat{\Omega}_i ) - \mathbf{I}_d = \mathbf{P} \mathbf{J} \mathbf{P}^\top
\end{equation*}
where $\mathbf{J}$ is block-diagonal matrix of type
\begin{equation*}
    \mathbf{J} = \begin{bmatrix} \mathbf{J}_1 & & \\ & \dots & \\ & & \mathbf{J}_p \end{bmatrix}
\end{equation*}
where for each $j$ $\mathbf{J}_j = \exp (\eta \mathbf{E}_j) - \mathbf{I}$ where $\mathbf{I}$ is either $\mathbf{I}_1$ or $\mathbf{I}_2$. Hence, for each $j$ $\mathbf{J}_j$ is either $\begin{bmatrix} 0 \end{bmatrix}$ or a two-dimensional matrix of the form
\begin{equation*}
    \mathbf{J}_j = \begin{bmatrix} \cos ( \eta \mu_j ) - 1 & - \sin ( \eta \mu_j ) \\ \sin ( \eta \mu_j ) & \cos ( \eta \mu_j ) - 1 \end{bmatrix}
\end{equation*}

Denote by $\mathcal{J}$ the set of indices $j$ from $\{ 1, \dots p \}$ which correspond to two-dimensional blocks of $\mathbf{E}$ and $\mathbf{J}$. Then
\begin{align*}
    \| \exp ( \eta \widehat{\Omega}_i ) - \mathbf{I}_d \|_\mathcal{F}^2 &= \| \mathbf{P} \mathbf{J} \mathbf{P}^\top \|_\mathcal{F}^2 = \| \mathbf{J} \|_\mathcal{F}^2 = 2 \sum_{j \in \mathcal{J}} \biggl( \sin^2 ( \eta \mu_j ) + (\cos ( \eta \mu_j ) - 1)^2 \biggr) = 4 \sum_{j \in \mathcal{J}} \biggl( 1 - \cos ( \eta \mu_j ) \biggr) \\
    &\leq 2 \sum_{j \in \mathcal{J}} ( \eta \mu_j )^2 = \eta^2 \| \mathbf{E} \|_\mathcal{F}^2 = \eta^2 \| \widehat{\Omega}_i \|_\mathcal{F}^2
\end{align*}
where we use the inequality $1 - \cos x \leq \frac{x^2}{2}$. Therefore we can rewrite Equation \ref{lipsch_g1} as
\begin{align*}
    | g' (\eta) - g' (0) | \leq \biggl( (2 \sqrt{d} + 1) L + \| \nabla F (\mathbf{I}_d) \|_\mathcal{F} \biggr) \| \widehat{\Omega}_i \|^2_\mathcal{F} \cdot | \eta | \leq L_g \cdot | \eta |
\end{align*}
where $L_g = \biggl( (2 \sqrt{d} + 1) L + \| \nabla F (\mathbf{I}_d) \|_\mathcal{F} \biggr) \| \widehat{\Omega}_i \|^2_\mathcal{F}$. We further deduce:
\begin{equation}
    g(\eta) - g(0) - \eta g'(0) = \int^\eta_0 \biggl( g'(\tau) - g'(0) \biggr) d \tau \geq - \int^\eta_0 \biggl| g'(\tau) - g'(0) \biggr| d \tau \geq - \int^\eta_0 L_g | \tau | d \tau = - \frac{\eta^2}{2} L_g \label{g_bound}
\end{equation}
We unfold $g$'s definition, put $\eta = \eta_i$ and rewrite \ref{g_bound} as follows:
\begin{equation} \label{f_bound}
    \eta_i \langle \nabla F (\mathbf{X}_i ) , \widehat{\Omega}_i \mathbf{X}_i \rangle_{\mathrm{e}} \leq F(\mathbf{X}_{i + 1}) - F(\mathbf{X}_i) + \frac{\eta^2_i}{2} L_g
\end{equation}
Recall that from $\widehat{\Omega}_i$'s definition we have that $\mathbb{E} \widehat{\Omega}_i = \Omega_i = \nabla_\mathcal{O} F (\mathbf{X}_i) \mathbf{X}_i^\top$. By taking expectation w.r.t. random $\widehat{\Omega}_i$ sampling at $i$'s step from both sides of Equation \ref{f_bound} we obtain that
\begin{equation*}
    \eta_i \langle \nabla F (\mathbf{X}_i ) , \nabla F_\mathcal{O} (\mathbf{X}_i ) \rangle_{\mathrm{e}} \leq \mathbb{E} F(\mathbf{X}_{i + 1}) - F(\mathbf{X}_i) + \frac{\eta^2_i}{2} \mathbb{E} L_g
\end{equation*}
Since the Riemannian gradient can be expressed as $\nabla F_\mathcal{O} (\mathbf{X}_i) = (\nabla F (\mathbf{X}_i) \mathbf{X}_i^\top - \mathbf{X}_i \nabla F (\mathbf{X}_i)^\top ) \mathbf{X}_i$, we have that
\begin{align*}
    \| \nabla F_\mathcal{O} (\mathbf{X}_i) \|_\mathcal{F}^2 &= \| \nabla F (\mathbf{X}_i) \mathbf{X}_i^\top - \mathbf{X}_i \nabla F (\mathbf{X}_i)^\top \|_\mathcal{F}^2 = \mathrm{tr} \biggl( ( \nabla F (\mathbf{X}_i) \mathbf{X}_i^\top - \mathbf{X}_i \nabla F (\mathbf{X}_i)^\top )^\top \nabla F (\mathbf{X}_i) \mathbf{X}_i^\top \biggr) \\
    &+ \mathrm{tr} \biggl( (\mathbf{X}_i \nabla F (\mathbf{X}_i)^\top - \nabla F (\mathbf{X}_i) \mathbf{X}_i^\top)^\top \mathbf{X}_i \nabla F (\mathbf{X}_i)^\top \biggr) \\
    &= \mathrm{tr} \biggl( \mathbf{X}_i^\top ( \nabla F (\mathbf{X}_i) \mathbf{X}_i^\top - \mathbf{X}_i \nabla F (\mathbf{X}_i)^\top )^\top \nabla F (\mathbf{X}_i) \biggr) + \mathrm{tr} \biggl( \nabla F (\mathbf{X}_i)^\top ( \nabla F (\mathbf{X}_i) \mathbf{X}_i^\top - \mathbf{X}_i \nabla F (\mathbf{X}_i)^\top ) \mathbf{X}_i \biggr) \\
    &= 2 \langle (\nabla F (\mathbf{X}_i) \mathbf{X}_i^\top - \mathbf{X}_i \nabla F (\mathbf{X}_i)^\top ) \mathbf{X}_i, \nabla F (\mathbf{X}_i) \rangle_\mathrm{e} = 2 \langle \nabla F_\mathcal{O} (\mathbf{X}_i), \nabla F (\mathbf{X}_i) \rangle_\mathrm{e}
\end{align*}
where we use that $\mathrm{tr} (\mathbf{A}^\top \mathbf{B}) = \mathrm{tr} (\mathbf{B}^\top \mathbf{A})$ and $\mathrm{tr} (\mathbf{A} \mathbf{B}) = \mathrm{tr} (\mathbf{B} \mathbf{A})$. Hence
\begin{align}
    \eta_i \| \nabla F_\mathcal{O} (\mathbf{X}_i ) \|^2_\mathcal{F} &\leq 2 \biggl( \mathbb{E} F (\mathbf{X}_{i + 1}) - F (\mathbf{X}_i) \biggr) + \eta^2_i \mathbb{E} L_g \\
    &\leq 2 \biggl( \mathbb{E} F (\mathbf{X}_{i + 1}) - F (\mathbf{X}_i) \biggr) + \eta^2_i \biggl( (2 \sqrt{d} + 1) L + \| \nabla F (\mathbf{I}_d) \|_\mathcal{F} \biggr) \sigma^2 \label{grad_bound}
\end{align}

By taking expectation of Equation \ref{grad_bound} w.r.t. $\widehat{\Omega}_i$ random sampling at steps $i=\overline{0..T}$ and summing over all these steps one arrives at
\begin{align*}
    \sum_{i = 0}^T \eta_i \mathbb{E} \| \nabla F_\mathcal{O} (\mathbf{X}_i ) \|^2_\mathcal{F} &\leq 2 \mathbb{E} \biggl( F(\mathbf{X}_{i + 1}) - F (\mathbf{X}_0) \biggr) + \sigma^2 \biggl( (2 \sqrt{d} + 1) L + \| \nabla F (\mathbf{I}_d) \|_\mathcal{F} \biggr) \sum_{i = 0}^T \eta^2_i \\
    &\leq 2 \biggl( F^* - F(\mathbf{X}_0) \biggr) + \sigma^2 \biggl( (2 \sqrt{d} + 1) L + \| \nabla F (\mathbf{I}_d) \|_\mathcal{F} \biggr) \sum_{i = 0}^T \eta^2_i
\end{align*}
Finally we use that
\begin{equation*}
    \biggl[ \sum_{i = 0}^T \eta_i \biggr] \cdot \min_{i = \overline{0..T}} \mathbb{E} \| \nabla F_\mathcal{O} (\mathbf{X}_i ) \|^2_\mathcal{F} \leq \sum_{i = 0}^T \eta_i \mathbb{E} \| \nabla F_\mathcal{O} (\mathbf{X}_i ) \|^2_\mathcal{F}
\end{equation*}
and divide by $\sum_{i = 0}^T \eta_i$ to conclude the proof.
\end{proof}

\subsection{Stochastic Optimization on the Orthogonal Group vs Recent Results on Givens Rotations for ML}
\label{app:givens}

There is an interesting relation between algorithms for stochastic optimization on the orthogonal group $\mathcal{O}(d)$ proposed by us and some results results about applying Givens rotations in machine learning. 

\paragraph{Givens Neural Networks:} In \cite{kamas} the authors propose neural network architectures, where matrices of connections are encoded as \textbf{trained products} of Givens rotations. They demonstrate that such architectures can be effectively used for neural network based policies in reinforcement learning and furthermore provide the compactification of the parameters that need to be learned. Notice that such matrices of connections correspond to consecutive steps of the matching-based optimizers/integrators proposed by us. This points also to an idea of neural ODEs that are constrained to evolve on compact manifolds (such as an orthogonal group).

\paragraph{Approximating Haar measure:} Approximating Haar measure on the orthogonal group $\mathcal{O}(d)$ was recently shown to have various important applications in machine learning, in particular for kernel methods \cite{chor_rowland}
and in general in the theory of Quasi Monte Carlo sequences \cite{gcmc}. Some of the most effective methods conduct approximations through products of random Givens matrices \cite{chor_rowland}. It turns out that we can think about this problem through the lens of matrix differential equations encoding flows evolving on $\mathcal{O}(d)$. Consider the following DE on the orthogonal group:
\begin{equation}
\label{rand_de}
\mathbf{\dot{X}}(t) = \mathbf{X}(t)\Omega_{\mathrm{rand}}(t)
\end{equation}
with an initial condition: $\mathbf{X}(0) \in \mathcal{O}(d)$. It turns out that when $\Omega_{\mathrm{rand}}(t)$ is "random enough" (one can take for instance Gaussian skew-symmetric matrices with large enough standard deviations of each entry or random walk skew-symmetric matrices, where each entry of the upper triangular part is an independent long enough random walk on a discrete $1\mathrm{d}$-lattice $\{0,1,-1,2,-2,...\}$), the above differential equation describes a flow on $\mathcal{O}(d)$ such that for $T \rightarrow \infty$ the distribution of $\mathbf{X}(t)$ converges to Haar measure.
Equation \ref{rand_de} is also connected to heat kernels on $\mathcal{O}(d)$.

Interestingly, if we use our stochastic matching-based methods for integrating such a flow, we observe that the solution is a product of \textbf{random} Givens rotations. Furthermore, these products tend to have the property that vertices/edges corresponding to different Givens rotations do not appear for consecutive elements that often as for the standard method (for instance, every block of Givens rotations corresponding to one step of the integration uses different edges since they correspond to a valid matching). We do believe that such property helps to obtain even stronger mixing properties in comparison to standard mechanism. Finally, these products of Givens rotations can be seen right now as a special instantiation of a much more general mechanism, since nothing prevents us from using our methods with $s>2$ rather than $s=2$ to conduct integration. That provides a convenient way to trade-off accuracy of the estimator versus its speed. We leave detailed analysis of the applications of our methods in that context to future work.


\end{document}